\pdfoutput=1
\documentclass[sigconf]{acmart}

\usepackage{booktabs} 
\usepackage{courier}
\usepackage{graphicx}
\usepackage{graphics}
\usepackage{url}
\usepackage{multirow}
\usepackage{color}
\usepackage{subfig}
\usepackage{balance}
\usepackage{amssymb}
\usepackage{amsmath}
\usepackage{bbm}
\usepackage{algorithmicx}
\usepackage{algorithm}
\usepackage{algpseudocode}
\usepackage{pgfplots}
\usepackage{bm}
\usepackage{cleveref}
\usepackage{hyperref}
\def\NoNumber#1{{\def\alglinenumber##1{}\State #1}\addtocounter{ALG@line}{-1}}

\crefformat{section}{\S#2#1#3} 
\crefformat{subsection}{\S#2#1#3}
\crefformat{subsubsection}{\S#2#1#3}
\newenvironment{customlegend}[1][]{%
	\begingroup
	\csname pgfplots@init@cleared@structures\endcsname
	\pgfplotsset{#1}%
}{%
\csname pgfplots@createlegend\endcsname
\endgroup
}%
\def\addlegendimage{\csname pgfplots@addlegendimage\endcsname}
\newcommand{\EFone}{\mathrm{EF1}}

\newcommand{\MMS}{\mathrm{MMS}}
\newcommand{\ALG}{\textsc{Alg}}

\DeclareMathOperator*{\argmax}{arg\,max}

\newcommand{\Mod}[1]{\ (\mathrm{mod}\ #1)}
\newtheorem{theorem}{Theorem}[section]
\newtheorem{lemma}{Lemma}[theorem]
\newtheorem{proposition}[theorem]{Proposition}

\usepackage{enumitem}
\newcommand{\subscript}[2]{$#1 _ #2$}

\AtBeginDocument{%
  \providecommand\BibTeX{{%
    \normalfont B\kern-0.5em{\scshape i\kern-0.25em b}\kern-0.8em\TeX}}}

\setcopyright{acmcopyright}
\setcopyright{iw3c2w3}

\acmConference[WWW '20]{Proceedings of The Web Conference 2020}{April 20--24, 2020}{Taipei, Taiwan}
\acmBooktitle{Proceedings of The Web Conference 2020 (WWW '20), April 20--24, 2020, Taipei, Taiwan}
\acmPrice{}
\acmDOI{10.1145/3366423.3380196}
\acmISBN{978-1-4503-7023-3/20/04}

\begin{document}

\title{FairRec: Two-Sided Fairness for Personalized Recommendations in Two-Sided Platforms}

\author{Gourab K Patro}
\authornote{Authors contributed equally to this work.}
\affiliation{IIT Kharagpur, India}

\author{Arpita Biswas}
\authornotemark[1]
\affiliation{Indian Institute of Science, India}

\author{Niloy Ganguly}
\affiliation{IIT Kharagpur, India}

\author{Krishna P. Gummadi}
\affiliation{MPI-SWS, Germany}

\author{Abhijnan Chakraborty}
\affiliation{MPI-SWS, Germany}
\renewcommand{\shortauthors}{}

\begin{abstract}
	We investigate the problem of fair recommendation in the context of two-sided online platforms, comprising customers on one side and producers on the other.  Traditionally, recommendation services in these platforms have focused on maximizing customer satisfaction by tailoring the results according to the personalized preferences of individual customers. 
	However, our investigation reveals that such customer-centric design may lead to unfair distribution of exposure among the producers, which may adversely impact their well-being.
	On the other hand, a producer-centric design might become unfair to the customers.
	Thus, we consider fairness issues that span both customers and producers. Our approach involves a novel mapping of the fair recommendation problem to a constrained version of the problem of fairly allocating indivisible goods. Our proposed {\em FairRec} algorithm guarantees at least Maximin Share ($\MMS$) of exposure for most of the producers and Envy-Free up to One item ($\EFone$) fairness for every customer.
	Extensive evaluations over multiple real-world datasets show the effectiveness of {\em FairRec} in ensuring two-sided fairness while incurring a marginal loss in the overall recommendation quality.
\end{abstract}

\ccsdesc[500]{Information systems~Recommender systems}

\keywords{Fair Recommendation, Two-Sided Markets, Fair Allocation, Maximin Share, Envy-Freeness}
\maketitle

\section{introduction}
\label{introduction}
Popular online platforms such as Netflix, Amazon, Yelp, Spotify, Google Local provide recommendation services to help their customers browse through the enormous product spaces.
By providing these services, the platforms control the interaction between the two stakeholders, namely (i) {\bf producers} of goods and services (e.g., movies on Netflix, products on Amazon, restaurants on Yelp, artists on Spotify) and  
(ii) {\bf customers} who consume them. 
These platforms have traditionally focused on maximizing customer satisfaction by tailoring the results according to the personalized preferences of individual customers, largely ignoring the interest of the producers.
Several recent studies have shown how such customer-centric designs may undermine the well-being of the producers \cite{abdollahpouri2019multi,burke2017multisided,edelman2017racial,graham2017digital,hannak2017bias}.
As more and more people are depending on two-sided platforms to earn a living,
recently platforms have started showing interest in creating fair marketplaces for all the stake holders due to multiple reasons:
{\bf (i) legal obligation} (e.g., labor bill for the welfare of drivers on Uber and Lyft~\cite{nyt2019uber}, fair marketplace laws for e-commerce~\cite{fair_marketplace}),
{\bf (ii) social responsibility or voluntary commitment} (e.g., equality of opportunity to all gender groups in LinkedIn~\cite{geyik2019fairranking}, commitment of non-discrimination to hosts and guests by AirBnb~\cite{airbnb_commitment}),
{\bf (iii) business requirement/model} (e.g., minimum business guarantee by AirBnb to attract hosts~\cite{airbnb_min_guarantee}).

In this paper, our focus is on the fairness of personalized recommendation services deployed on the two-sided platforms.
Traditionally, platforms employ various state-of-the-art data-driven methods (e.g., neighborhood-based methods~\cite{ning2015comprehensive}, latent factorization methods~\cite{liang2016factorization,koren2009matrix}, etc.) to estimate the relevance scores of every product-customer pairs, and then recommend $k$ most relevant products to the corresponding customers. 
While such top-$k$ recommendations achieve high customer utility, our investigation on real-world datasets reveals that they can create a huge disparity in the exposure of the producers (detailed in \cref{observations}), which is unfair for the producers, and may also hurt the platforms in the long term.

In these platforms, \textit{exposure often determines the economic opportunities (revenues) for the producers} who depend on it for their livelihood.
For instance, high exposure on {\em Google Maps} can increase the footfall in a local business, thereby increasing their revenue. High exposure on {\em YouTube, Spotify or Last.fm} can increase the traffic to a content producer's channel, and hence help them earn better platform-royalties or advertisement revenues.
On the other hand, if only a few producers get most of the exposure, then the other producers would struggle on the platform, which will force them to either quit or switch to other platforms~\cite{leaving_yelp,leaving_amz,leaving_uber}. 
This, in turn, may limit the choices for the customers, degrading the overall experience on the platform.
Thus, it is important to reduce exposure inequalities. 
Naive ways of reducing inequality (e.g., producer-centric poorest-$k$: selecting $k$ least exposed producers) may result in loss and disparity in customer utilities (\cref{observations}), making it inefficient as well as unfair to the customers.
We postulate that while being fair to the producers, the platforms should also attempt to fairly distribute the loss in utility among all the customers. 

Considering this, we tackle the challenging task of ensuring two-sided fairness while giving personalized recommendations. 
Motivated by a vast literature in social choice theory, we map this problem to the problem of fairly allocating indivisible goods (\cref{mapping}). However, due to various constraints pertaining to the requirements of recommender systems, 
the problem becomes an interesting extension to the existing fair allocation problem---find an allocation that guarantees minimum exposure (upper bounded by maximin share of exposure or MMS) for the producers, and envy-free up to one item (EF1)~\cite{budish2011combinatorial} 
for the customers (proofs in \cref{sec:theorems}). The $\MMS$ guarantee ensures that each agent receives a value which is at least their \emph{maximin share} threshold, defined in \cref{eq:mms}; whereas, $\EFone$ ensures that every agent values her allocation at least as much as any other agent's allocation after (hypothetically) removing the most valuable item from the other agent's allocated bundle.\\

\noindent {\bf Contributions:}
(i) we consider two-sided notions of fairness which not only relate to social or judicial precepts but also to the long-term sustainability of two-sided platforms (\cref{motivation});
(ii) we design an algorithm, {\it FairRec} (\cref{algorithm}), exhibiting the desired two-sided fairness by mapping the fair recommendation problem to a fair allocation problem (\cref{mapping}); 
moreover, it is agnostic to the specifics of the data-driven model (that estimates the product-customer relevance scores) which makes it more scalable and easy to adapt;
(iii) in addition to the theoretical guarantees (\cref{sec:theorems}), 
extensive experimentation and evaluation over multiple real-world datasets deliver strong empirical evidence 
on the effectiveness of our proposal~(\cref{experiments}).
\section{Background and Related work}
\label{related}
We briefly survey related works in two directions: (i) fairness in multi-stakeholder platforms, and (ii) fair allocation of goods.

\noindent{\bf Fairness in Two-Sided Platforms: }
With the increasing popularity of multi-sided platforms, recently multiple researchers have looked into the issues of unfairness and biases in such platforms.
For example, \citet{edelman2017racial} investigated the possibility of racial bias in guest acceptance by Airbnb hosts, 
~\citet{lambrecht2019algorithmic} studied gender-based discrimination in career ads. While these works deal with {\it group fairness}, ~\citet{serbos2017fairness} proposed an {\it envy free} tour package recommendations on travel booking sites, ensuring {\it individual fairness for customers}. 

On producer fairness, \citet{hannak2017bias} studied racial and gender bias in freelance marketplaces.
In a social experiment, \citet{Salganik854} found that the existing popular producers often acquire most of the visibility while new but good ones starve for visibility. 
\citet{biega2018equity} considered individual producer fairness in ranking in gig-economy platforms.
\citet{kamishima2014correcting} and \citet{abdollahpouri2017controlling} reduced popularity bias among producers while \citet{patro2020incremental} addressed fairness issues arising due to frequent updates of platforms.
However, these papers did not study the trade-off between producer and customer fairness, and the cost of achieving one over the other. 

A few papers have considered group-fairness among the producers and customers. Abdollahpouri et al.~\cite{abdollahpouri2019multi} and \citet{burke2017multisided} categorized different types of multi-stakeholder platforms and their desired group fairness properties, \citet{chakraborty2017fair} and \citet{suhr2019two} presented mechanisms for two-sided fairness in matching problems.
In contrast, our paper addresses individual fairness for both producers and customers, which also answers the question of the long-term sustainability of two-sided platforms.

\noindent {\bf Fair Allocation of Goods:}
The problem of fair allocation (popularly known as the cake-cutting problem) has been studied extensively in the area of computational social choice theory. The classical notions of fairness for this problem are envy-freeness (EF)~\cite{foley1967resource,varian1974equity,stromquist1980cut} and proportional fair share (PFS)~\cite{steinhaus1948problem}. 
Recent literature on practical applications of fair allocation~\cite{brandt2016handbook,endriss2017trends} have focused on the problem of allocating \textit{indivisible} goods 
in budgeted course allocation~\cite{budish2011combinatorial}, balanced graph partition~\cite{fair-graph}, or allocation of cardinality constrained group of resources~\cite{biswas2018fair}. In such instances, no feasible allocation may satisfy EF or PFS fairness guarantees. Thus, the notable work of Budish~\cite{budish2011combinatorial} defined analogous fairness notions which are appropriate for \textit{indivisible} goods---namely, envy-freeness up to one good ($\EFone$) and maximin share guarantee ($\MMS$). A rich literature has focused on providing existence and algorithmic guarantees for computing fair allocations~\cite{amanatidis2018comparing,barman2018groupwise,amanatidis2015approximation,procaccia2014fair,kurokawa2016can,bouveret2014characterizing,caragiannis2016unreasonable,biswas2019matroid,bilo2018almost}.
In this work, we map the problem of fair recommendation to a fair allocation problem, which leads to an interesting extension of previously studied problems owing to the specific constraints pertaining to recommendations (detailed in~\cref{mapping}).
\section{Preliminaries}
\label{notions}
Next, we define the terminology and notations used in the paper.

\subsection{Producers and Products}
In platforms like Google Maps and Yelp, a producer typically owns one product (restaurant or shop); whereas in multimedia platforms like Spotify, YouTube, and Netflix, the songs/videos generated by the artists are the products, and one producer can list many products; same is true for ecommerce platforms such as Amazon and Flipkart.
To generalize all such two-sided platforms, we consider products and producers to be equivalent, and use the terms `product' and `producer' interchangeably.
Even for platforms where a producer can have multiple products, ensuring fairness at the product level would also ensure fairness for individual producers; here, fairness can be ensured by making the exposure proportional to the producer's portfolio size.

\subsection{Notations}
\label{notations}
Let $U$ and $P$ be the sets of customers and producers respectively, where $|U|=m$, and $|P|=n$.
Let $k$ be the number of products to be recommended to every customer. $R_u\subset P$ represents the set of $k$ products recommended to customer $u$; $|R_u|=k$.

\subsection{Relevance of Products}
\label{relevance}
The \textit{relevance} of a product $p$ to customer $u$, denoted as $V_u(p)$, represents the likelihood that $u$ would like the product $p$.
Formally, relevance is a function from the set of customers and products to the real numbers $V:~ U \times P \rightarrow \mathbb{R}$.
Usually, the relevance scores are predicted using various data-driven methods, and $V_u(p)$ is a proxy for the utility gained by $u$ if product $p$ is recommended to her.

\subsection{Customer Utility}
\label{customer_utility}
The utility of a recommendation $R_u$ to a customer $u$ is proportional to the sum of relevance scores of products in $R_u$. 
Thus, recommending the $k$ most relevant products will give the maximum possible utility.
Let $R_u^*$ be the set of top-$k$ relevant products for $u$.
We use a normalized form of customer utility from $R_u$, defined as: $\phi^\text{}_u(R_u)=\frac{\sum_{p\in R_u} V_u(p)}{\sum_{p\in R_u^*} V_u(p)}$.

\subsection{Producer Exposure}
\label{exposure}
Exposure of a producer/product $p$ is the total amount of attention that $p$ receives from all the customers to whom $p$ has been recommended. In this paper, we assume a uniform attention model\footnote{This being the first work on two-sided-fair-recommendation posed as a fair-allocation problem, we focused on a basic setting without \textit{position bias}~\cite{agarwal2019estimating}, where customers pay more attention to the top ranked products than the lower ranked ones.} where customers pay similar attention to all $k$ recommended products, and express the exposure of a product $p$ 
as $E_p = \sum_{u\in U}\mathbbm{1}_{R_u}(p)$, where $\mathbbm{1}_{R_u}(p)$ is $1$ if $p\in R_u$, and $0$ otherwise. The sum of exposures of all the products is $\sum_{p\in P}E_p=m\times k$.
\begin{figure}[h]
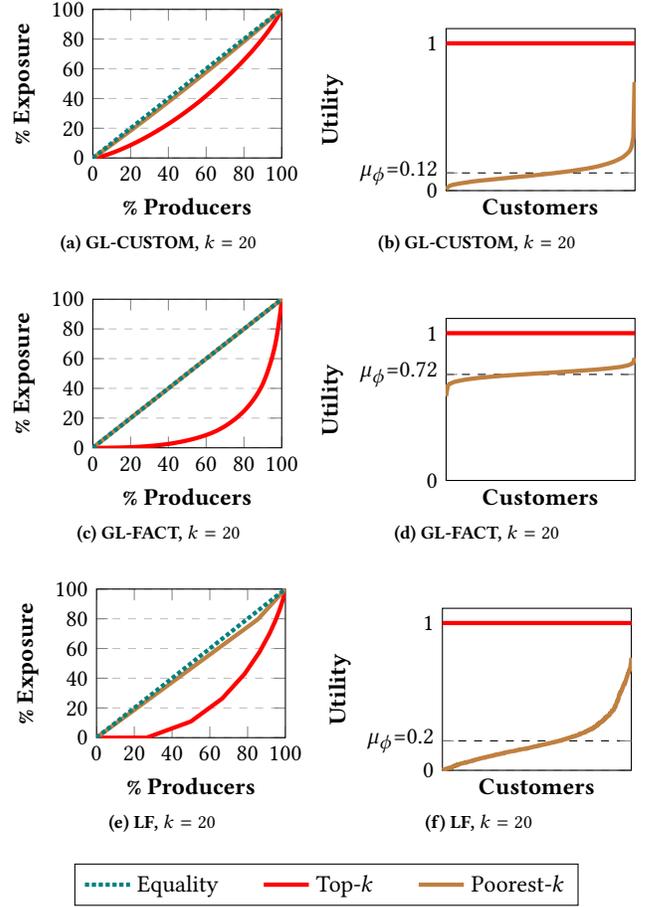

	\center
	{
		\subfloat[{GL-CUSTOM, $k=20$}]{\pgfplotsset{height=0.2\textwidth,width=0.23\textwidth,compat=1.9}\input{lorentz_gl1_k_20}\label{fig:lorentz_gl1_k_20}}
		\hfil
		\subfloat[{GL-CUSTOM, $k=20$}]{\pgfplotsset{height=0.21\textwidth,width=0.23\textwidth,compat=1.9}
	\begin{tikzpicture}
	\begin{axis}[
	xlabel={\bf Customers},
	ylabel={\bf Utility},
	xmin=0, xmax=11200,
	ymin=0, ymax=1.1,
	xmajorticks=false,
	ytick={0,1},
	ymajorgrids=true,
	grid style=dashed,
	extra y ticks = 0.12,
	extra y tick labels={$\mu_\phi$=0.12},
	extra y tick style={grid=major,major grid style={draw=black}}
	]
	\addplot[
	color=brown,
	mark=.,
	ultra thick
	]
	coordinates {(1, 0.00699065359215553)
		(100, 0.02365370236665052)
		(200, 0.031348768738631336)
		(300, 0.03619860392667935)
		(400, 0.03950164819959321)
		(500, 0.041573692556280704)
		(600, 0.04390045303731797)
		(700, 0.04615362876109925)
		(800, 0.04812161241613646)
		(900, 0.05009212829717307)
		(1000, 0.051752334563724474)
		(1100, 0.053477126237112545)
		(1200, 0.05529541678341434)
		(1300, 0.056988929195536144)
		(1400, 0.05901961204395505)
		(1500, 0.06060044664229926)
		(1600, 0.06204749250399109)
		(1700, 0.06359172844768123)
		(1800, 0.06504083928020231)
		(1900, 0.06613356793327167)
		(2000, 0.06731830230253162)
		(2100, 0.06860689824762949)
		(2200, 0.06966383893722658)
		(2300, 0.07076253181252172)
		(2400, 0.0717344013980872)
		(2500, 0.07293413911053867)
		(2600, 0.07417606947175977)
		(2700, 0.07540713214240541)
		(2800, 0.07642954141634514)
		(2900, 0.07767972912834209)
		(3000, 0.07884546193305346)
		(3100, 0.07987688557993174)
		(3200, 0.08088044756897744)
		(3300, 0.08229948825684377)
		(3400, 0.0833526959526827)
		(3500, 0.08443160699086145)
		(3600, 0.08572762689329652)
		(3700, 0.08686341734269769)
		(3800, 0.08795557892902887)
		(3900, 0.0891006623366116)
		(4000, 0.09035847973023273)
		(4100, 0.09146626270716267)
		(4200, 0.09245714199497383)
		(4300, 0.09345229837275704)
		(4400, 0.09448371544157234)
		(4500, 0.09559456490461134)
		(4600, 0.09656959749637857)
		(4700, 0.09751185972734738)
		(4800, 0.09884379413792725)
		(4900, 0.09994730468359347)
		(5000, 0.1010307635518633)
		(5100, 0.10220248467464244)
		(5200, 0.10342429893791187)
		(5300, 0.10480250092087)
		(5400, 0.10616748951150971)
		(5500, 0.10716027138760532)
		(5600, 0.10827743363312631)
		(5700, 0.10955700355594936)
		(5800, 0.1109090930350941)
		(5900, 0.11218959841629578)
		(6000, 0.11341775591879968)
		(6100, 0.11470861975160765)
		(6200, 0.11589924843111897)
		(6300, 0.11723006313829475)
		(6400, 0.11852740653199403)
		(6500, 0.12016544829019615)
		(6600, 0.12149304232960145)
		(6700, 0.12283909819326115)
		(6800, 0.12429390672985872)
		(6900, 0.12558531471657983)
		(7000, 0.12690564264291138)
		(7100, 0.12840696462203885)
		(7200, 0.12996018978114077)
		(7300, 0.13133786556613797)
		(7400, 0.13298930237494502)
		(7500, 0.1346216339046593)
		(7600, 0.13642187063003966)
		(7700, 0.13848280226519952)
		(7800, 0.14018172678138896)
		(7900, 0.1417735724526403)
		(8000, 0.14371495247192462)
		(8100, 0.1454521268664993)
		(8200, 0.14717273263426925)
		(8300, 0.14902173268013225)
		(8400, 0.150956146114409)
		(8500, 0.1530188505464684)
		(8600, 0.1550672813447508)
		(8700, 0.15754707613257987)
		(8800, 0.1598060096372817)
		(8900, 0.16216906538095324)
		(9000, 0.16498608944512017)
		(9100, 0.16779901447077858)
		(9200, 0.17066847583208763)
		(9300, 0.17389895835854413)
		(9400, 0.17693645978554434)
		(9500, 0.18039728874760783)
		(9600, 0.18357315911803543)
		(9700, 0.18732508290360858)
		(9800, 0.191421666233705)
		(9900, 0.1953273035403226)
		(10000, 0.1996465524211722)
		(10100, 0.20521411572606743)
		(10200, 0.21024773644897818)
		(10300, 0.21591750909129537)
		(10400, 0.22191118034923818)
		(10500, 0.22849917082213206)
		(10600, 0.23797179386817444)
		(10700, 0.2477927425551374)
		(10800, 0.25837835453185265)
		(10900, 0.27209516519131144)
		(11000, 0.29960532979963694)
		(11100, 0.3626376975508132)
		(11172, 0.7371606409311067)
		
		};
	\addplot[red,domain=0:11200, ultra thick, unbounded coords=discard] plot (x,1);

	\end{axis}
	\end{tikzpicture}
\label{fig:cdf_gl1_k_20}}
		\vfil
		\subfloat[{GL-FACT, $k=20$}]{\pgfplotsset{height=0.2\textwidth,width=0.23\textwidth,compat=1.9}\input{lorentz_gl2_k_20}\label{fig:lorentz_gl2_k_20}}
		\hfil
		\subfloat[{GL-FACT, $k=20$}]{\pgfplotsset{height=0.21\textwidth,width=0.23\textwidth,compat=1.9}
	\begin{tikzpicture}
	\begin{axis}[
	xlabel={\bf Customers},
	ylabel={\bf Utility},
	xmin=0, xmax=11200,
	ymin=0, ymax=1.1,
	xmajorticks=false,
	ytick={0,1},
	ymajorgrids=true,
	grid style=dashed,
	extra y ticks = 0.72,
	extra y tick labels={$\mu_\phi$=0.72},
	extra y tick style={grid=major,major grid style={draw=black}}
	]
	\addplot[
	color=brown,
	mark=.,
	ultra thick
	]
	coordinates {(1, 0.5716463485377943)
		(100, 0.6446243693198144)
		(200, 0.655499052808033)
		(300, 0.6601673882154061)
		(400, 0.6652789351886351)
		(500, 0.6693571439761095)
		(600, 0.6724214041157962)
		(700, 0.6747060032255529)
		(800, 0.677469612837208)
		(900, 0.6795795454409838)
		(1000, 0.6814649391243266)
		(1100, 0.6837589922266403)
		(1200, 0.6858250797252692)
		(1300, 0.6874719063023633)
		(1400, 0.6889676190940265)
		(1500, 0.690147056449289)
		(1600, 0.6915708439852754)
		(1700, 0.6931073156795856)
		(1800, 0.6946451454363173)
		(1900, 0.6960896590158745)
		(2000, 0.6971280829337911)
		(2100, 0.6982165309237534)
		(2200, 0.699424424405482)
		(2300, 0.7004025430687688)
		(2400, 0.7012830315368936)
		(2500, 0.70254706547989)
		(2600, 0.7035621110382448)
		(2700, 0.7046701549671601)
		(2800, 0.7055760527563212)
		(2900, 0.7065322947150235)
		(3000, 0.707570536382276)
		(3100, 0.7084594755591309)
		(3200, 0.7093553896617959)
		(3300, 0.7103787313385965)
		(3400, 0.7112055859792998)
		(3500, 0.7119662035377213)
		(3600, 0.7129465276921146)
		(3700, 0.7138434376321117)
		(3800, 0.7145815622356135)
		(3900, 0.7153442274875497)
		(4000, 0.7160408360680195)
		(4100, 0.7168055496701026)
		(4200, 0.7176722237282713)
		(4300, 0.7183466583175274)
		(4400, 0.719157770112441)
		(4500, 0.7198950164035123)
		(4600, 0.7206726549586847)
		(4700, 0.72133014808704)
		(4800, 0.72210951049682)
		(4900, 0.7227994920685586)
		(5000, 0.7235509914485821)
		(5100, 0.7243981383248203)
		(5200, 0.7252033937029162)
		(5300, 0.7259363733817149)
		(5400, 0.7266592192290229)
		(5500, 0.7273789754358204)
		(5600, 0.7281339831453499)
		(5700, 0.7286781003635237)
		(5800, 0.7293983193194429)
		(5900, 0.7300166226937065)
		(6000, 0.730609187680674)
		(6100, 0.731333213526066)
		(6200, 0.7319697879809575)
		(6300, 0.7326462766307079)
		(6400, 0.7333677831642385)
		(6500, 0.7339745688487541)
		(6600, 0.7348224881759652)
		(6700, 0.7355345045680579)
		(6800, 0.7363676134409529)
		(6900, 0.737139445867156)
		(7000, 0.7378335997538731)
		(7100, 0.738601983657842)
		(7200, 0.7393081766675749)
		(7300, 0.740024755622166)
		(7400, 0.7407470807547856)
		(7500, 0.7415195720802534)
		(7600, 0.7422808560748914)
		(7700, 0.7430373421536933)
		(7800, 0.7439051922792225)
		(7900, 0.7446188185291827)
		(8000, 0.7454877728452013)
		(8100, 0.746376286179203)
		(8200, 0.7472787660961365)
		(8300, 0.7480514471768398)
		(8400, 0.7489268354317942)
		(8500, 0.7498919667364402)
		(8600, 0.750571561567856)
		(8700, 0.7513618645265177)
		(8800, 0.7524264707683403)
		(8900, 0.7533619603009847)
		(9000, 0.7542981654846534)
		(9100, 0.75530105928607)
		(9200, 0.7562488833827367)
		(9300, 0.7572774598085809)
		(9400, 0.7583831153701196)
		(9500, 0.7596775369720257)
		(9600, 0.760745001280151)
		(9700, 0.7620136582966036)
		(9800, 0.7632706545234457)
		(9900, 0.7646907273042226)
		(10000, 0.7662562807250712)
		(10100, 0.7676465119058667)
		(10200, 0.7689764068798788)
		(10300, 0.7705573208544113)
		(10400, 0.7723575246359348)
		(10500, 0.7745546034077689)
		(10600, 0.7771143534136962)
		(10700, 0.7794858657017216)
		(10800, 0.7826237894259)
		(10900, 0.78672657044187)
		(11000, 0.7922328937996642)
		(11100, 0.800787978256486)
		(11172, 0.8323389247193983)
		
		};
	\addplot[red,domain=0:11200, ultra thick, unbounded coords=discard] plot (x,1);

	\end{axis}
	\end{tikzpicture}
\label{fig:cdf_gl2_k_20}}
		\vfil
		\subfloat[{LF, $k=20$}]{\pgfplotsset{height=0.2\textwidth,width=0.23\textwidth,compat=1.9}\input{lorentz_lf_k_20}\label{fig:lorentz_lf_k_20}}
		\hfil
		\subfloat[{LF, $k=20$}]{\pgfplotsset{height=0.21\textwidth,width=0.23\textwidth,compat=1.9}\input{cdf_lf_k_20}\label{fig:cdf_lf_k_20}}
		\vfil
		\subfloat{\pgfplotsset{width=.5\textwidth,compat=1.9}
			\begin{tikzpicture}
			\begin{customlegend}[legend entries={{Equality},{Top-$k$},{Poorest-$k$}},legend columns=3,legend style={/tikz/every even column/.append style={column sep=0.5cm}}]
			\addlegendimage{teal,ultra thick, densely dotted,sharp plot}
			\addlegendimage{red,mark=.,ultra thick,sharp plot}
			\addlegendimage{brown,mark=.,ultra thick,sharp plot}
			\end{customlegend}
			\end{tikzpicture}}
	}
	
	\caption{Lorenz curves show high inequality among producer exposures with the top-k recommendation. While poorest-k provides almost equal exposures, it introduces huge loss and disparity in individual customer utilities.}
	\label{fig:lorenz_n_cdf}
	
\end{figure}
\section{Need for two-sided fairness in personalized recommendations}
\label{motivation}
Traditionally, the goal of personalized recommendation has been to recommend products that would be most relevant to a customer.
This task typically requires learning the relevance scoring functions ($V$), and several state-of-the-art data-driven methods~\cite{ning2015comprehensive,liang2016factorization} have been developed to estimate the product-customer relevance values.
Once these values are obtained, the standard practice, across several recommender systems, is to recommend the top-$k$ ($k$=size of recommendation) relevant products to corresponding customers.
While the above approach is followed to maximize the satisfaction of individual customers, it can adversely affect the producers in a two-sided platform, as we explore next.

\subsection{Datasets}
\label{dataset}
We consider the impact of customer-centric top-$k$ recommendations on producer exposures using real-world datasets.
We use a state-of-the-art relevance scoring model (a very widely used latent factorization method~\cite{liang2016factorization}) and also a dataset-specific custom relevance scoring model over the datasets.
\subsubsection{\bf Google Local Ratings Dataset (GL)}~\\
Google Local is a service to find nearby places on Google Maps (as Google Nearby feature) platform.
We use the Google Local dataset released by \citet{he2017translation}, which contains data about customers, local businesses (producers), and their locations (geographic coordinates), ratings, etc.
We consider the active customers located in New York City and the business entities within $5$ miles radius of Manhattan area with at least $10$ reviews. The resulting dataset contains $11172$ customers, $855$ businesses and $25686$ reviews.
We consider the following two relevance scoring functions ($V$).\\
\noindent \textbf{\textit{A. GL-CUSTOM:}}
We use a custom relevance scoring function: $V(u,p)=\frac{rating(p)}{distance(u,p)}$, where $rating(p)$ is the average rating of the producer (local business) and $distance(u,p)$ is the distance between customer $u$ and producer $p$.\\
\noindent \textbf{\textit{B. GL-FACT:}}
Here we use the state-of-the-art latent factorization model~\cite{liang2016factorization,koren2009matrix} to predict the relevance scores from the ratings.\\

\subsubsection{\bf Last.fm Dataset (LF)}~\\
We use the {\tt Last.fm} dataset released by \citet{Cantador:RecSys2011}, which contains $1892$ customers, $17632$ artists (producers), and $92834$ records of play counts (the number of times a customer has played songs from an artist).
We again use a latent factorization model~\cite{liang2016factorization,koren2009matrix} to find out the relevance scores from the play counts. 
\subsection{Adverse Impact of Customer-Centric Recommendation}
\label{observations}
We simulate top-$k$ ($k=20$) recommendation on all three datasets, 
and calculate the exposure different producers get.
Figures-\ref{fig:lorentz_gl1_k_20},\ref{fig:lorentz_gl2_k_20},\ref{fig:lorentz_lf_k_20} are the Lorenz curves for producer exposures.
In Exposure Lorenz curves, the cumulative fraction of total exposure is plotted against the cumulative fraction of the number of corresponding producers (ranked in increasing order of their exposures). The extent to which the curve goes below a straight diagonal line (or an equality mark) indicates the degree of inequality in the exposure distribution.
We observe that the Lorenz curves for top-$k$ recommendations are far below the equal exposure marks, revealing that for conventional top-$k$ recommendation, $50$\% least exposed producers get 
only $32$\%, $5$\%, and $11$\% of total available exposure ($m\cdot k$) in {\em GL-CUSTOM}, {\em GL-FACT}, and {\em LF} datasets, respectively.

\noindent {\bf Huge disparity is in nobody's interest: } In two-sided platforms, the exposure determines the economic opportunities.
Thus, low exposure on a platform often puts many producers at huge losses, forcing them to leave the platform;
this may result in fewer available choices for the customers, thereby degrading the overall quality of the platform.
Thus, highly skewed exposure distribution of the customer-centric top-$k$ recommendation not only makes it unfair to the producers but also questions the long-term sustainability of the platforms. 
Thus, there is a need to be fair to the producers while designing recommender systems.

\subsection{Pitfalls of Naive Solution}
A naive approach to reduce inequality in producer exposures is to implement a producer-centric recommendation (poorest-$k$):
recommend the least-$k$ exposed products to the customer at any instant. Such producer-centric scheme makes the exposure of all the producers nearly equal, as seen in figures-\ref{fig:lorentz_gl1_k_20},\ref{fig:lorentz_gl2_k_20},\ref{fig:lorentz_lf_k_20}: 
Lorenz curves for poorest-$k$ recommendations are closer to the diagonal than those of top-$k$ recommendations. 
However, such poorest-$k$ recommendation decreases overall customer utilities (as seen in figures-\ref{fig:cdf_gl1_k_20},\ref{fig:cdf_gl2_k_20},\ref{fig:cdf_lf_k_20}).
Moreover, the poorest-$k$ introduces disparity among individual customer utilities, where some customers may suffer much higher loss than other customers, making it unfair to them. 
\subsection{Desiderata of Fair Recommendation}
\label{fairness_properties}
Intuitively, we need the following fairness properties to be satisfied by the recommendation to be fair to both producers and customers.\\

\noindent \textbf{A. Producer Fairness: }
Mandating a uniform exposure distribution over the producers can be too harsh on the system;
it may heavily hamper the quality of the recommendation, and might also kill the existing competition by discouraging the producers from improving the quality of their products or services.
Instead, we propose to ensure a minimum exposure guarantee for every producer such that no producer starves for exposure. 
The proposal is comparable to the fairness of {\it minimum wage guarantee} (e.g., as required by multiple legislations in the US, starting from Fair Labor Standards Act 1938 to Fair Minimum Wage Act 2007)~\cite{pollin2008measure,green2010minimum,falk2006fairness}).
Ensuring minimum wage does not itself guarantee equality of income;
however it has been found to decrease income inequality~\cite{lin2016effects,engbom2018earnings}.
The exact value of the minimum exposure guarantee ({\small $\overline{E}$}) can be decided by the respective platforms.\\

\noindent \textbf{B. Customer Fairness: }
As maintaining producer fairness can cause an overall loss in customer utility;
that loss should be fairly distributed among the customers.
To ensure this, products need to be recommended in a way such that no customer can gain extra utility by exchanging her set of recommended products with another customer (a property called {\it envy-freeness}, as detailed next).
\section{Re-imagining Fair Recommendation as Fair Allocation}
\label{mapping}
Given a set of items (say, $\mathcal{P}$), a set of agents (say, $\mathcal{U}$), and valuations $\mathcal{V}$ (how much an agent values an item),
the \textbf{\textit{fair allocation problem}} aims at distributing the items \textit{fairly} among the agents. In the discrete version of this problem, the items are \textit{discrete} (no item can be broken into pieces) and \textit{non-shareable} (no item can be allocated to multiple agents).
If $\mathcal{P}$ contains several copies of the same item, 
each copy can be thought of as non-shareable and discrete.
The goal is to find 
a non-shareable and discrete allocation ($\mathcal{A}:=\{(A_u)_{u\in \mathcal{U}}: A_u\subseteq \mathcal{P}\}$) while ensuring 
fairness properties.

\subsection{Notions of Fairness in Allocation}
The classical fairness notions, such as envy-freeness\footnote{An allocation is said to satisfy \textit{envy-freeness} if the bundle of items allocated to each agent is as valuable to her as the bundle allocated to any other agent~\cite{foley1967resource,varian1974equity,stromquist1980cut}.} (EF) and proportional-fair-share\footnote{An allocation is said to satisfy \textit{proportional-fair-share} if each agent receives a bundle of value at least $1/|\mathcal{U}|^{th}$ of her total value for all the items~\cite{steinhaus1948problem}.} (PFS), may not be achievable in most instances of the problem.
For example, if there are two agents and one item, the item will be allocated to one of the agents, and the zero allocation to the other agent would violate both EF and PFS.
Thus, for \textit{discrete} items, relaxations of EF and PFS have been considered.
Two such well-studied notions of fairness in the discrete fair allocation literature are (i)~\emph{envy freeness up to one item} ($\EFone$) and (ii)~the \textit{maximin share guarantee} ($\MMS$), defined by~\citet{budish2011combinatorial}. Since then, these have been extensively studied in various settings for providing existential and algorithmic guarantees~\cite{amanatidis2018comparing,barman2018groupwise,brandt2016handbook,amanatidis2015approximation,procaccia2014fair,kurokawa2016can,bouveret2014characterizing,caragiannis2016unreasonable,biswas2018fair,biswas2019matroid,fair-graph,bilo2018almost}.
We now formally state these fairness notions:
\begin{itemize}[leftmargin=*]
	\item An allocation $\mathcal{A}$ is $\EFone$ iff for every pair of agents $u , w \in \mathcal{U}$ there exists an item $p \in A_w$ such that $\mathcal{V}_u(A_u) \geq \mathcal{V}_u(A_w\setminus\{p\})$.
	
	\item An allocation is said to satisfy $\MMS$ if each agent receives a value greater than or equal to their \emph{maximin share} threshold. This threshold for an agent $u$ is defined as
	\begin{equation}
		\MMS_u= \max_{\mathcal{A}}\  \min_{w\in \mathcal{U}}\ \mathcal{V}_u(A_w).\label{eq:mms}
	\end{equation}
	In other words, $\MMS_u$ is the maximum value that the agent can guarantee for herself if she were to allocate $\mathcal{P}$ into $|\mathcal{U}|$ bundles and then, from those bundles, receive the minimum valued one.
	Formally, an allocation $\mathcal{A}$ satisfy $\MMS$ fairness iff for all agents $u \in \mathcal{U}$, we have $\mathcal{V}_u(A_u) \geq \MMS_u$.
\end{itemize}

\subsection{Fair Recommendation to Fair Allocation}
We propose to see the problem of fair recommendation as a {\em fair allocation} problem.
The set of products $P$ can be thought of as the set of items $\mathcal{P}$ (there can be multiple copies of individual products);
similarly the set of customers $U$ as the set of agents $\mathcal{U}$, and the relevance scoring function $V$ as the valuations $\mathcal{V}$;
now the task of recommending products to customers is the same as allocating items in $\mathcal{P}$ to agents in $\mathcal{U}$ with certain constraints.
\begin{itemize} [leftmargin=*]
	\item {\bf Setting $\mathcal{P}$ for Producer Fairness:}
	As the total exposure of the platform remains limited ($k\cdot |U|$), the maximum guarantee on minimum possible exposure for the producers is $\left\lfloor\frac{k\cdot |U|}{|P|}\right\rfloor$ (this refers to the $\mathrm{MMS}$ value for the producers).
	Thus, based on desired exposure guarantee $\overline{E}$ ($\overline{E}\leq \mathrm{MMS}$), we can make $\overline{E}$ copies of each product (for producer fairness as in \cref{fairness_properties}) to construct $\mathcal{P}$.
	\item {\bf Fair Allocation of $\mathcal{P}$ among $\mathcal{U}$:}
	Once $\mathcal{P}$ is set according to desired producer fairness, the entire task of fair recommendation boils down to the allocation of $\mathcal{P}$ among $\mathcal{U}$ while ensuring $\EFone$ for agents/customers (for customer fairness, \cref{fairness_properties}).
\end{itemize}

\subsection{Extending the Conventional Fair Allocation Problem}
Traditionally, \textit{fair allocation} literature aims at defining and ensuring \textit{fairness} among the agents while allocating all the items belonging to the set $\mathcal{P}$. However, in the {\it fair recommendation} problem, along with customer fairness, the challenge is to attain producer (or product) fairness by providing a minimum exposure guarantee (say, each product should be allocated to at least $\ell$ different customers). Thus, achieving producer fairness is the same as creating at least $\ell$ copies of each product and ensuring that all the copies are allocated, along with a feasibility constraint which enforces that no customer gets more than one copy of the same product. This extension of the problem---where all the items are grouped into disjoint categories and no agent receives more than a pre-specified number of items from the same category---is called cardinality constrained fair allocation problem, proposed 
in~\cite{biswas2018fair}. In this paper, we consider a novel extension of the cardinality constrained problem by adding another constraint enforcing that exactly $k$ items are allocated to each customer. This requires tackling hierarchical feasibility constraints---an upper bound cardinality constraint of one on each product and a cardinality constraint of $k$ on the total number of allocated products. Moreover, this additional feasibility constraint makes it difficult to decide how many copies of which product should be made available for a total of ($k\cdot |U|$) allocations, satisfying the feasibility constraints as well as the fairness requirement. 
Thus, unlike the fair allocation problem, we consider no restriction on the number of copies of each product that are made available. All these contrast points, along with two-sided fairness guarantees make fair recommendation an interesting extension of the fair allocation problem.

\section{FairRec: An Algorithm to ensure two-sided fairness}
\label{algorithm}
In this section, we provide a polynomial-time algorithm {\em FairRec}, for finding an allocation $\mathcal{A}$ which satisfies the desired two-sided fairness described in \cref{fairness_properties} (we prove the theoretical guarantees in ~\cref{sec:theorems}).
Note that we consider only the case of $k<|P|$, and leave the trivial case of $k=|P|$ and the infeasible case of $k>|P|$ out of consideration.
Also, we consider $|P|\leq k\cdot |U|$, otherwise, at least ($|P|-k\cdot |U|$) producers can not be allocated to any customer.
{
	\begin{algorithm}[h]
		{\raggedright
			{
				{\raggedright{\bf Input:} Set of customers $U=[m]$, set of distinct products $P=[n]$,   recommendation set size $k$ (such that $k<n$ and $n\leq k\cdot m$), and the relevance scores $V_u(p)$.}\\
				{\raggedright{\bf Output:} A two-sided fair recommendation.}
			}
			\caption{{\em FairRec} ($U, P, k, V$)}
			\label{alg:two-sided}
			\begin{algorithmic}[1]
				
				\State Initialize allocation $\mathcal{A}^0=(A^0_1, \ldots, A^0_m)$ with $A^0_i~\leftarrow~\emptyset$ for each customer $i \in [m]$.
				
				\NoNumber{}
				\NoNumber{\textbf{First Phase:}}
				\State Fix an (arbitrary) ordering of the customers $\sigma = \left(\sigma(1), \sigma(2),\ldots, \sigma(m)\right)$.
				\State Initialize set of feasible products $F_u \leftarrow P$ for each $u\in [m]$.
				\State Set $\ell \leftarrow \left\lfloor \frac{m\times k}{n}\right\rfloor$ denoting number of copies of each product.	
				\State Initialize each component of the vector $S=(S_1, \ldots, S_n)$ with $S_j\leftarrow \ell$, $\forall j\in [n]$, this stores the number of available copies of each product.
				\State Set $T\leftarrow \ell\times n$, total number of items to be allocated.
				\State $[\mathcal{B}, F, x]\leftarrow$Greedy-Round-Robin$(m,n,S,T,V,\sigma,F)$.
				\State Assign $\mathcal{A}\leftarrow \mathcal{A}\cup \mathcal{B}$.
				\NoNumber{}
				\NoNumber{\textbf{Second Phase:}}
				\State Set $\Lambda=|A_{\sigma((x)\Mod{m}+1)}|$ denoting the number of items allocated to the customer subsequent to $x$, according to the ordering $\sigma$. 
				\If{ $\Lambda<k$}
				\State Update each component of the vector $S=(S_1,\ldots,S_n)$ with the value $m$ in order to allow allocating any product to any customer.
				\State Set $T\leftarrow 0$.
				\If{ $x<m$ }
				\State Set $\sigma'(i)\leftarrow \sigma((i+x-1)\Mod{m}+1)$ for all $i\in[m]$.
				\State $\sigma\leftarrow \sigma'$.
				\State $T\leftarrow (m-x)$.
				\State Update $\Lambda\leftarrow \Lambda+1$.
				\EndIf				
				\State $T\leftarrow T+m(k-\Lambda)$ total number of items to be allocated.
				\State $[\mathcal{C},F,x]\leftarrow$Greedy-Round-Robin$(m,n,S,T,V,\sigma,F)$.
				\State Assign $\mathcal{A}\leftarrow \mathcal{A}\cup \mathcal{C}$.
				\EndIf
				
				\State Return $\mathcal{A}$.
			\end{algorithmic}}
		\end{algorithm}
	}
	
	The algorithm executes in two phases.
	The first phase ensures $\EFone$ among all the $m$ customers (Lemma \ref{lemma:customer_fairness}) and tries to provide a minimum guarantee on the exposure of the producers (Lemma~\ref{lemma:producer_fairness}).
	However, the first phase may not allocate exactly $k$ products to all the $m$ users, which is then ensured by the second phase while simultaneously maintaining $\EFone$ for customers. 
	
	The {\bf first phase} creates $\ell=\left\lfloor\frac{mk}{n}\right\rfloor$ copies of each product (note that $\left\lfloor\frac{mk}{n}\right\rfloor$ is the maximin value of any producer while $mk$ slots are allocated among $n$ producers).
	It then initializes each component of the vector $S$ of size $|P|$ to $\ell$ to ensure that at most $\ell$ copies from each product are allocated in the first phase.
	Feasible sets $F_u$ for each customer $u$ are then initialized to ensure that each customer receives at most one copy of the same product.
	Then, assuming an arbitrary ordering $\sigma$ of customers, $\ALG~\ref{alg:greedy}$ is executed and the allocation $\mathcal{B}$ is obtained.

	The {\bf second phase} checks if all the customers have received exactly $k$ products (by looking at the number of products allocated to the customer $x+1$ which is next-in-sequence to the last allocated customer $x$ of the first phase). If the customer $x+1$ has received $k$ products, then no further allocation is required;
	if not, then $\ALG$~\ref{alg:greedy} is called again with a new ordering obtained by $x$ left-cyclic rotations of $\sigma$.
	The remaining number of items is stored in $T$ which are to be allocated among the customers.
	Also, each component of the vector $S$ is updated to $|U|$ to allow allocating any feasible product without any limit on the available number of copies.
	The second phase retains $\EFone$ fairness among the customers. 
	
	Both phases use a modified version of the Greedy-round-robin algorithm ($\ALG~\ref{alg:greedy}$)~\cite{caragiannis2016unreasonable,biswas2018fair}:
	it follows the ordering $\sigma$ in a round-robin fashion (i.e., it selects customers, one after the other, from $\sigma(1)$ to $\sigma(m)$), and iteratively assigns to the selected customer her most desired unallocated product (feasibility maintained by the vector $S$ and sets $F_u$ and ties are broken arbitrarily).
	This process is repeated over several rounds until one of the two disjoint conditions occur: (i)~$T==0$: a total of $T$ allocations have occurred, or (ii)~$p==\emptyset$: no feasible product available (for the current customer $\sigma(i)$, we have $F_{\sigma(i)} \cap \{p: S_p\neq 0\}= \emptyset$). Finally, it returns an allocation $B_1,\ldots,B_m$ with each $B_u\subseteq[n]$ for all $u\in[m]$.
	
	{
		\begin{algorithm}[t]
			{\raggedright
				{
					{\raggedright {\bf Input :} Number of customers $m$, number of producers $n$, an array with number of available copies of each product $S$, total number of available products $T>0$, relevance scores $V_u(p)$ and feasible product set $F_u$ for each customer, and an ordering $\sigma$ of $[m]$.}\\
					{\raggedright {\bf Output:} An allocation of $T$ products among $m$ customers, the residual feasible set $F_u$ and the last allocated index $x$.
					}
					\caption{Greedy-Round-Robin ($m,n,S,T,V,\sigma,F$)}
					\label{alg:greedy}
					\begin{algorithmic}[1]
						\State Initialize allocation $\mathcal{B}=(B_1, \ldots, B_m)$ with $B_i~\leftarrow~\emptyset$ for each customer $i \in [m]$.	
						\State Initiate $x \leftarrow m$.
						\State Initiate round $r\leftarrow 0$.
						\While{ true }
						\State Set $r\leftarrow r+1.$
						\For{$i =1 \mbox{ to } m$}
						\State Set $p \in \underset{p'\in F_{\sigma(i)}:(S_p\neq 0)}{\argmax} V_{\sigma(i)}(p')$ 
						\If{$p == \emptyset$}
						\State Set $x=i-1$ only if $i\neq 1$.
						\State \textbf{go to} Step $22$.
						\EndIf
						\State Update $B_{\sigma(i)} \leftarrow B_{\sigma(i)} \cup p$. 
						\State Update $F_{\sigma(i)} \leftarrow F_{\sigma(i)} \setminus p$.
						\State Update $S_p\leftarrow S_p-1$.
						\State Update $T\leftarrow T-1$.					
						\If{ $T==0$ }
						\State $x=i$.
						\State \textbf{go to} Step $22$.
						\EndIf
						\EndFor
						\EndWhile			
						\State Return $\mathcal{B}=(B_1,\ldots,B_m)$, $F=(F_1,\ldots,F_m)$ and index $x$.
					\end{algorithmic}}}
				\end{algorithm}
			}
			
\section{Theoretical Guarantees}\label{sec:theorems}
In this section, we provide a few important properties of $\ALG$~\ref{alg:greedy} in Proposition~\ref{proposition:greedy}. Later, we establish the fairness guarantees and time complexity of our proposed algorithm \textit{FairRec} in Theorem~\ref{theorem:two-sided} using Lemma~\ref{lemma:customer_fairness}, \ref{lemma:producer_fairness} and \ref{lemma:polytime}.
Note that, for all the proofs, we have fixed $\alpha=1$ or $\overline{E}=$MMS.

\begin{proposition}\label{proposition:greedy}
	The allocation obtained by the $\mathrm{Greedy}$-$\mathrm{Round}$-$\mathrm{Robin}$ algorithm ($\ALG$~\ref{alg:greedy}) exhibits the following four properties:
	\begin{enumerate}[label=(\subscript{P}{{\arabic*}})]
		\item for any two indices $x$ and $y$, where $x < y$, the customer $\sigma(x)$ (who appears earlier than $\sigma(y)$ according to the ordering $\sigma$) does not envy customer $\sigma(y)$, i.e., $V_{\sigma(x)} (B_{\sigma(x)} ) \geq V_{\sigma(x)} (B_{\sigma(y)})$.
		\item the allocation $\mathcal{B}$ obtained by $\ALG$~\ref{alg:greedy} is $\EFone$.
		\item each customer is allocated at most one item from the same producer, thus ensuring the cardinality constraint is satisfied for each producer (category).
		\item for any two customers, say $u$ and $w$, the allocation $\mathcal{B}$ obtained by $\ALG$~\ref{alg:greedy} satisfies the following: $-1\leq \left(|B_u|-|B_w|\right) \leq 1$.    
	\end{enumerate}
\end{proposition}

\begin{proof}
	The properties $P_1$ and $P_2$ have been observed by ~\citet{biswas2018fair} and ~\citet{caragiannis2016unreasonable}, respectively. For completeness, we repeat the arguments towards these two properties. Let $x$ and $y$ be two indices, such that $1\leq x<y\leq m$. At each round $r$, the customer $\sigma(x)$ chooses her most desired product among all the unallocated items before customer $\sigma(y)$. Hence, $V_{\sigma(x)}(p^r_{\sigma(x)})~\geq~V_{\sigma(x)}(p^r_{\sigma(y)})$, where $p^r_{\sigma(x)}$ and $p^r_{\sigma(y)}$ denote the items assigned to customer $\sigma(x)$ and $\sigma(y)$, respectively. Thus, over all the rounds,  $\sum_r V_{\sigma(x)}(p^r_{\sigma(x)}) \geq \sum_r V_{\sigma(x)}(p^r_{\sigma(y)})$ which implies that $V_{\sigma(x)}(B_{\sigma_{x}})\geq V_{\sigma(x)}(B_{\sigma_{y}})$ and thus the property $P_1$ holds. 
	
	Property $P_2$ states that if $\sigma(y)$ envies $\sigma(x)$, it will not violate $\EFone$ property (note: we already saw in $P_1$ that $\sigma(x)$ does not envy $\sigma(y)$). Now observe that, the value $V_{\sigma(y)}(\cdot)$ of the item allocated to customer $\sigma(y)$ in the $r$th round is at least that of the item allocated to customer $\sigma(x)$ in the $(r+1)$th round. Let $R$ denote the total number of rounds, then the following holds: 
	\begin{align} 
		&V_{\sigma(y)}(p^r_{\sigma(y)}) \geq V_{\sigma(y)}(p^{r+1}_{\sigma(x)})\quad \mbox{ for all } r \in \{1,\ldots,R-1\}\nonumber \\ 
		\Rightarrow &\sum_{r=1}^{R-1} V_{\sigma(y)}(p^r_{\sigma(y)}) \geq \sum_{r=1}^{R-1} V_{\sigma(y)}(p^{r+1}_{\sigma(x)}) \nonumber\\
		\Rightarrow &V_{\sigma(y)}(B_{\sigma(y)}) \geq V_{\sigma(y)}(B_{\sigma(x)}) -V_{\sigma(y)}(p^1_{\sigma(x)})\label{eq:EF1_greedy}
	\end{align}
	Equation~\ref{eq:EF1_greedy} shows that the customer $\sigma(y)$ stops envying $\sigma(x)$ when only one item is (hypothetically) removed from $\mathcal{B}_{\sigma(x)}$ (namely, $p^1_{\sigma(x)}$). Thus, the allocation $\mathcal{B}$ is $\EFone$, i.e., $P_2$ holds.
	
	The property $P_3$ is satisfied by the use of the feasible sets $F_u$ for each customer $u$. Each $F_u$ contains the set of producers who have not yet been allocated to the customer $u$. At any round $r$, step $7$ of $\ALG~\ref{alg:greedy}$ selects the most relevant producer among the producers who had not been allocated to $u$ in any earlier rounds $r'<r$. Once, a producer $p$ is allocated to a customer $u$, step $9$ of $\ALG~\ref{alg:greedy}$ removes $p$ from $F_u$. Thus, each customer is allocated at most one item from the same producer.
	
	The property $P_4$ states that, for any pair of customers $u$ and $w$, the number of allocated items $|B(u)|$ and $|B(v)|$, differ by at most $1$. It is straightforward to see that, except for the last feasible round, all customers are allocated exactly one item at each round. Thus, all the customers receive the same number of allocations until the second last feasible round. In the last feasible round, some customers may not get any allocation (if there is no available feasible product) and thus may receive one item less than the others.
\end{proof}

\noindent We now state the main theorem (Theorem~\ref{theorem:two-sided}) that establishes the fairness guarantees of our proposed algorithm.

\begin{theorem}\label{theorem:two-sided}
	Given $n$ producers, the proposed polynomial time algorithm, $\mathrm{FairRec}$, returns an $\EFone$ allocation among $m$ customers while allocating exactly $k$ items  to each customer, when $k<n\leq mk$. Moreover, it ensures non-zero exposure among all the $n$ producers and $\MMS$ guarantee among at least $n-k$ producers. 
\end{theorem}

\begin{proof}
	We prove the fairness guarantees of customers and producers in Lemma~\ref{lemma:customer_fairness} and \ref{lemma:producer_fairness}, respectively. In Lemma~\ref{lemma:polytime}, we show that $\mathrm{FairRec}$ executes in polynomial time.
\end{proof}

\begin{lemma}\label{lemma:customer_fairness}
	Given $n$ producers, $m$ customers, and a positive integer $k$ (such that $k<n\leq mk$), $\mathrm{FairRec}$ returns an $\EFone$ allocation among $m$ customers while allocating exactly $k$ items to each customer.
\end{lemma}

\begin{proof}
	To prove this, we show that both phases of $\mathrm{FairRec}$ satisfy $\EFone$. Since $\ALG$~\ref{alg:greedy} guarantees $\EFone$ (by property $P_2$), the allocation $\mathcal{A}$ at step $9$ of $\mathrm{FairRec}$ is $\EFone$. Thus, for any two customers $u$ and $w$, there exists an item $j\in B_w$ such that $V_u(B_u)\geq V_u(B_w)-V_w(j)$. Next, the second phase creates $|U|$ copies of each product and calls $\ALG~\ref{alg:greedy}$ to obtain the allocation $\mathcal{C}$. Note that the second phase assigns the most valued item to each customer at each round, that is, it allocates top-$\Lambda_u$ feasible producers to each customer, where $\Lambda_u = k-|B_u|$. Thus, $V_u(C_u)\geq V_u(C_w)$. Thus,  $V_u(B_u)+V_u(C_u)\geq V_u(B_w)-V_w(j) + V_u(C_w)$, which implies $\EFone$: $V_u(B_u\cup C_u)\geq V_u(B_w\cup C_w)-V_u(j)$. This completes the proof that $\mathrm{FairRec}$ ensures $\EFone$ among all the customers while recommending exactly $k$ products to each customer.
\end{proof}

\begin{lemma}\label{lemma:producer_fairness}
	Given $n$ producers, $m$ customers, and a positive integer $k$ (such that $k<n\leq mk$), $\mathrm{FairRec}$ ensures non-zero exposure among all the $n$ producers. Moreover, it assures $\MMS$-fairness among at least $n-k$ producers. \end{lemma}

\begin{proof}
	We first prove that the first phase guarantees non-zero exposure for producers. The allocation $\mathcal{B}$ obtained by $\ALG$~\ref{alg:greedy} in the first phase may have terminated for one of the two conditions
	\begin{enumerate}[leftmargin=*]
		\item $T==0$: this means that all the $\ell=\lfloor\frac{mk}{n}\rfloor$ copies of each producer have been allocated among all the customers. Thus, each producer receives exactly maximin threshold $\ell$. Hence $\MMS$ fairness is achieved by all the $n$ producers. 
		\item $p==\emptyset$: this happens when $T\neq 0$ and $\sum_{p\in F_u} S_p = 0$ for a customer $u$ (at termination). That is, $S_p=0$ for each producer $p\in F_u$. Thus, all $\ell=\lfloor\frac{mk}{n}\rfloor$ copies of the producers in the set $F_u$ have been allocated, and hence they attain $\MMS$ fairness. On the other hand, the producers in the set $B_u$ (the set recommended to customer $u$) is allocated to at least one producer. Thus, minimum value of $1$ is achieved by all the $n$ producers. Also, $|F_u|+|B_u|=n$ and $|B_u|\leq k$, implies that $|F_u|\geq n-k$. Therefore, at least $n-k$ producers attain $\MMS$-fairness.
	\end{enumerate}
	
	Since the thresholds are already satisfied in the first phase, adding more allocations in the second phase retains the threshold-based fairness guarantees. This completes the proof that $\mathrm{FairRec}$ ensures a non-zero exposure among all the $m$ producers and assures $\MMS$-fairness among at least $n-k$ producers.
\end{proof}

One consequence of Lemma~\ref{lemma:producer_fairness} is that, when $k$ is much lower than $n$, a large fraction of producers are guaranteed to attain $\MMS$ fairness. We formally state this property of $\mathrm{FairRec}$ algorithm in Corollary~\ref{lemma:corollary}.

\begin{corollary}\label{lemma:corollary}
	Given $n$ producers, a positive integer $k$, and $\beta\in(0,1)$ such that $k\leq\beta n$, $\mathrm{FairRec}$ ensures $\MMS$-fairness among at least $(1-\beta)n$ producers.
\end{corollary}

Finally, in Lemma~\ref{lemma:polytime}, we show that $\mathrm{FairRec}$ executes in polynomial time.

\begin{lemma}\label{lemma:polytime}
	The time complexity of $\mathrm{FairRec}$ has a worst case bound of $\mathcal{O}(mnk)$.
\end{lemma}
\begin{proof}
	The time complexity of $\mathrm{FairRec}$ is the same as the complexity of $\ALG$~\ref{alg:greedy}. Over the two phases, $\ALG$~\ref{alg:greedy} allocates $mk$ items. For each allocation, it finds the maximum possible feasible producer which can be done in at most $\mathcal{O}(n)$ time. Thus, the total time complexity of the algorithm is $\mathcal{O}(mnk)$.
\end{proof}

\begin{figure*}[t!]
	\center{
		\subfloat[{\bf GL-CUSTOM, Metric: \hyperref[metric_H]{$H$}}]{\pgfplotsset{width=0.33\textwidth,height=0.18\textwidth,compat=1.9}
			\begin{tikzpicture}
			\begin{axis}[
			xlabel={${k}$},
			ylabel={\hyperref[metric_H]{$H$}},
			xmin=0, xmax=20,
			ymin=0, ymax=1,
			xtick={0,5,10,15,20},
			ytick={0,0.2,0.4,0.6,0.8,1},
			ymajorgrids=true,
			grid style=dashed,
			ylabel style={rotate=-90},
			]
			\addplot[
			color=blue,
			mark=o,
			mark size=2pt
			]
			coordinates {(1, 1.0)
				(2, 1.0)
				(3, 1.0)
				(4, 1.0)
				(5, 1.0)
				(6, 1.0)
				(7, 1.0)
				(8, 1.0)
				(9, 1.0)
				(10, 1.0)
				(11, 1.0)
				(12, 1.0)
				(13, 1.0)
				(14, 1.0)
				(15, 1)
				(16, 1.0)
				(17, 1.0)
				(18, 1.0)
				(19, 1.0)
				(20, 1.0)			
				
			};
			
			\addplot[
			color=red,
			mark=x,
			mark size=2pt
			]
			coordinates {(1, 0.3742690058479532)
				(2, 0.38596491228070173)
				(3, 0.38011695906432746)
				(4, 0.39649122807017545)
				(5, 0.4070175438596491)
				(6, 0.4152046783625731)
				(7, 0.4245614035087719)
				(8, 0.4093567251461988)
				(9, 0.39415204678362575)
				(10, 0.39415204678362575)
				(11, 0.4070175438596491)
				(12, 0.41169590643274856)
				(13, 0.4269005847953216)
				(14, 0.43508771929824563)
				(15, 0.43391812865497076)
				(16, 0.447953216374269)
				(17, 0.43508771929824563)
				(18, 0.4456140350877193)
				(19, 0.45146198830409356)
				(20, 0.45497076023391814)			
			};
			
			\addplot[
			color=black,
			mark=star,
			mark size=2pt
			]
			coordinates {(1, 0.5356725146198831)
				(2, 0.5239766081871345)
				(3, 0.5274853801169591)
				(4, 0.5532163742690058)
				(5, 0.5356725146198831)
				(6, 0.5403508771929825)
				(7, 0.5309941520467836)
				(8, 0.5391812865497077)
				(9, 0.5321637426900585)
				(10, 0.5450292397660819)
				(11, 0.5298245614035088)
				(12, 0.5426900584795321)
				(13, 0.5228070175438596)
				(14, 0.5497076023391813)
				(15, 0.5157894736842106)
				(16, 0.5157894736842106)
				(17, 0.5076023391812865)
				(18, 0.5216374269005848)
				(19, 0.5017543859649123)
				(20, 0.5157894736842106)
								
			};
			\addplot[
			color=cyan,
			mark=triangle,
			mark size=2pt
			]
			coordinates{(1, 0.3742690058479532)
				(2, 0.3847953216374269)
				(3, 0.38011695906432746)
				(4, 0.3871345029239766)
				(5, 0.3906432748538012)
				(6, 0.4058479532163743)
				(7, 0.4093567251461988)
				(8, 0.4046783625730994)
				(9, 0.4269005847953216)
				(10, 0.4152046783625731)
				(11, 0.41754385964912283)
				(12, 0.4222222222222222)
				(13, 0.4198830409356725)
				(14, 0.4198830409356725)
				(15, 0.41403508771929826)
				(16, 0.42105263157894735)
				(17, 0.4093567251461988)
				(18, 0.40116959064327484)
				(19, 0.4128654970760234)
				(20, 0.40116959064327484)};
			
			\addplot[
			color=brown,
			mark=square,
			mark size=1.5pt
			]
			coordinates{(1, 1.0)
				(2, 1.0)
				(3, 1.0)
				(4, 1.0)
				(5, 1.0)
				(6, 1.0)
				(7, 1.0)
				(8, 1.0)
				(9, 1.0)
				(10, 1.0)
				(11, 1.0)
				(12, 1.0)
				(13, 1.0)
				(14, 1.0)
				(15, 1)
				(16, 1.0)
				(17, 1.0)
				(18, 1.0)
				(19, 1.0)
				(20, 1.0)};
			
			\end{axis}
			\end{tikzpicture}\label{fig:gl_1_H}}
		\hfil
		\subfloat[{\bf GL-FACT, Metric: \hyperref[metric_H]{$H$}}]{\pgfplotsset{width=0.33\textwidth,height=0.18\textwidth,compat=1.9}
			\begin{tikzpicture}
			\begin{axis}[
			xlabel={${k}$},
			ylabel={\hyperref[metric_H]{$H$}},
			xmin=0, xmax=20,
			ymin=0, ymax=1,
			xtick={0,5,10,15,20},
			ytick={0,0.2,0.4,0.6,0.8,1},
			ymajorgrids=true,
			grid style=dashed,
			ylabel style={rotate=-90},
			]
			\addplot[
			color=blue,
			mark=o,
			mark size=2pt
			]
			coordinates {(1, 1.0)
				(2, 1.0)
				(3, 1.0)
				(4, 1.0)
				(5, 1.0)
				(6, 1.0)
				(7, 1.0)
				(8, 1.0)
				(9, 1.0)
				(10, 1.0)
				(11, 1.0)
				(12, 1.0)
				(13, 1.0)
				(14, 1.0)
				(15, 1.0)
				(16, 1.0)
				(17, 1.0)
				(18, 1.0)
				(19, 1.0)
				(20, 1.0)

			};
			
			\addplot[
			color=red,
			mark=x,
			mark size=2pt
			]
			coordinates {(1, 0.15321637426900586)
				(2, 0.17192982456140352)
				(3, 0.18128654970760233)
				(4, 0.19181286549707602)
				(5, 0.1976608187134503)
				(6, 0.2)
				(7, 0.20350877192982456)
				(8, 0.21052631578947367)
				(9, 0.2128654970760234)
				(10, 0.21754385964912282)
				(11, 0.2222222222222222)
				(12, 0.22807017543859648)
				(13, 0.2304093567251462)
				(14, 0.2327485380116959)
				(15, 0.23625730994152047)
				(16, 0.23859649122807017)
				(17, 0.2409356725146199)
				(18, 0.24210526315789474)
				(19, 0.24678362573099416)
				(20, 0.25029239766081873)

			};
			
			\addplot[
			color=black,
			mark=star,
			mark size=2pt
			]
			coordinates {(1, 0.5625730994152047)
				(2, 0.5356725146198831)
				(3, 0.5309941520467836)
				(4, 0.5461988304093567)
				(5, 0.543859649122807)
				(6, 0.5333333333333333)
				(7, 0.5508771929824562)
				(8, 0.5321637426900585)
				(9, 0.5391812865497077)
				(10, 0.5274853801169591)
				(11, 0.5169590643274854)
				(12, 0.5508771929824562)
				(13, 0.5403508771929825)
				(14, 0.5380116959064327)
				(15, 0.5111111111111111)
				(16, 0.5076023391812865)
				(17, 0.52046783625731)
				(18, 0.5099415204678363)
				(19, 0.52046783625731)
				(20, 0.5087719298245614)

			};
			\addplot[
			color=cyan,
			mark=triangle,
			mark size=2pt
			]
			coordinates{(1, 0.15321637426900586)
				(2, 0.1567251461988304)
				(3, 0.17660818713450294)
				(4, 0.17660818713450294)
				(5, 0.18947368421052632)
				(6, 0.18947368421052632)
				(7, 0.19064327485380117)
				(8, 0.19064327485380117)
				(9, 0.1976608187134503)
				(10, 0.20350877192982456)
				(11, 0.19883040935672514)
				(12, 0.20584795321637428)
				(13, 0.20584795321637428)
				(14, 0.20584795321637428)
				(15, 0.20818713450292398)
				(16, 0.21754385964912282)
				(17, 0.22105263157894736)
				(18, 0.21403508771929824)
				(19, 0.2198830409356725)
				(20, 0.22105263157894736)
				
				};
			
			\addplot[
			color=brown,
			mark=square,
			mark size=1.5pt
			]
			coordinates{(1, 1.0)
				(2, 1.0)
				(3, 1.0)
				(4, 1.0)
				(5, 1.0)
				(6, 1.0)
				(7, 1.0)
				(8, 1.0)
				(9, 1.0)
				(10, 1.0)
				(11, 1.0)
				(12, 1.0)
				(13, 1.0)
				(14, 1.0)
				(15, 1)
				(16, 1.0)
				(17, 1.0)
				(18, 1.0)
				(19, 1.0)
				(20, 1.0)
				};
			
			\end{axis}
			\end{tikzpicture}\label{fig:gl_2_H}}
		\hfil
		\subfloat[{\bf LF, Metric: \hyperref[metric_H]{$H$}}]{\pgfplotsset{width=0.33\textwidth,height=0.18\textwidth,compat=1.9}
			\begin{tikzpicture}
			\begin{axis}[
			xlabel={${k}$},
			ylabel={\hyperref[metric_H]{$H$}},
			xmin=0, xmax=20,
			ymin=0, ymax=1,
			xtick={0,5,10,15,20},
			ytick={0,0.2,0.4,0.6,0.8,1},
			ymajorgrids=true,
			grid style=dashed,
			ylabel style={rotate=-90},
			]
			\addplot[
			color=blue,
			mark=o,
			mark size=2pt
			]
			coordinates {(1, 1.0)
				(2, 1.0)
				(3, 1.0)
				(4, 1.0)
				(5, 1.0)
				(6, 1.0)
				(7, 1.0)
				(8, 1.0)
				(9, 1.0)
				(10, 1.0)
				(11, 1.0)
				(12, 1.0)
				(13, 1.0)
				(14, 1.0)
				(15, 1.0)
				(16, 1.0)
				(17, 1.0)
				(18, 1.0)
				(19, 1.0)
				(20, 1.0)			
				
			};
			
			\addplot[
			color=red,
			mark=x,
			mark size=2pt
			]
			coordinates {(1, 1.0)
				(2, 1.0)
				(3, 1.0)
				(4, 1.0)
				(5, 1.0)
				(6, 1.0)
				(7, 1.0)
				(8, 1.0)
				(9, 1.0)
				(10, 0.5362409255898367)
				(11, 0.5631805807622504)
				(12, 0.5876247731397459)
				(13, 0.6126361161524501)
				(14, 0.6348117059891107)
				(15, 0.6557962794918331)
				(16, 0.67383166969147)
				(17, 0.6918670598911071)
				(18, 0.7057622504537205)
				(19, 0.4811705989110708)
				(20, 0.500397005444646)
								
			};
			
			\addplot[
			color=black,
			mark=star,
			mark size=2pt
			]
			coordinates {(1, 1.0)
				(2, 1.0)
				(3, 1.0)
				(4, 1.0)
				(5, 1.0)
				(6, 1.0)
				(7, 1.0)
				(8, 1.0)
				(9, 1.0)
				(10, 0.6607304900181489)
				(11, 0.6922640653357531)
				(12, 0.7257259528130672)
				(13, 0.75130444646098)
				(14, 0.7756919237749547)
				(15, 0.8038793103448276)
				(16, 0.8196460980036298)
				(17, 0.8402903811252269)
				(18, 0.8546960072595281)
				(19, 0.6057168784029038)
				(20, 0.6279491833030852)
								
			};
			\addplot[
			color=cyan,
			mark=triangle,
			mark size=2pt
			]
			coordinates{(1, 1.0)
				(2, 1.0)
				(3, 1.0)
				(4, 1.0)
				(5, 1.0)
				(6, 1.0)
				(7, 1.0)
				(8, 1.0)
				(9, 1.0)
				(10, 0.6189882032667876)
				(11, 0.6460980036297641)
				(12, 0.6801270417422868)
				(13, 0.7009414700544465)
				(14, 0.731510889292196)
				(15, 0.7518715970961888)
				(16, 0.7784142468239564)
				(17, 0.7908348457350273)
				(18, 0.8092105263157895)
				(19, 0.5516674228675136)
				(20, 0.5799115245009074)
				};
			\addplot[
			color=brown,
			mark=square,
			mark size=1.5pt
			]
			coordinates{(1, 1.0)
				(2, 1.0)
				(3, 1.0)
				(4, 1.0)
				(5, 1.0)
				(6, 1.0)
				(7, 1.0)
				(8, 1.0)
				(9, 1.0)
				(10, 1.0)
				(11, 1.0)
				(12, 1.0)
				(13, 1.0)
				(14, 1.0)
				(15, 1.0)
				(16, 1.0)
				(17, 1.0)
				(18, 1.0)
				(19, 1.0)
				(20, 1.0)
				};
			\end{axis}
			\end{tikzpicture}\label{fig:lf_H}}
		\hfil
		\subfloat[{\bf GL-CUSTOM, Metric: \hyperref[metric_Z]{$Z$}}]{\pgfplotsset{width=0.33\textwidth,height=0.18\textwidth,compat=1.9}
			\begin{tikzpicture}
			\begin{axis}[
			xlabel={${k}$},
			ylabel={\hyperref[metric_Z]{$Z$}},
			xmin=0, xmax=20,
			ymin=0.88, ymax=1,
			xtick={0,5,10,15,20},
			ytick={0.88,0.9,0.95,1},
			ymajorgrids=true,
			grid style=dashed,
			ylabel style={rotate=-90}
			]
			\addplot[
			color=blue,
			mark=o,
			mark size=2pt
			]
			coordinates {(1, 0.9999697942850987)
				(2, 0.9999831809801839)
				(3, 0.9999865163301157)
				(4, 0.999987594453314)
				(5, 0.9999889085145333)
				(6, 0.9999898586498918)
				(7, 0.9999891217852784)
				(8, 0.9999892164653388)
				(9, 0.9999893236670706)
				(10, 0.9999893945065103)
				(11, 0.9999903361298428)
				(12, 0.9999898188423013)
				(13, 0.9999903191318937)
				(14, 0.9999893396986967)
				(15, 1)
				(16, 0.9999998513263502)
				(17, 0.9999997079140589)
				(18, 0.9999995460519093)
				(19, 0.9999993515601615)
				(20, 0.9999991439597733)
				
			};
			
			\addplot[
			color=red,
			mark=x,
			mark size=2pt
			]
			coordinates {(1, 0.8932366664995447)
				(2, 0.9351977662770863)
				(3, 0.9521050438114186)
				(4, 0.9612716015809635)
				(5, 0.9668204595864393)
				(6, 0.9700320765011515)
				(7, 0.9722983101021868)
				(8, 0.9743340026187319)
				(9, 0.9754507827928609)
				(10, 0.9771059119075864)
				(11, 0.9779184031870628)
				(12, 0.9792144331737463)
				(13, 0.980227737057826)
				(14, 0.9810360783337784)
				(15, 0.9813752858758872)
				(16, 0.9820039178465332)
				(17, 0.9827499680020442)
				(18, 0.9829699977543527)
				(19, 0.9835640879872104)
				(20, 0.984106342444825)
				
			};
			
			\addplot[
			color=black,
			mark=star,
			mark size=2pt
			]
			coordinates {(1, 0.9941184797626098)
				(2, 0.9972655973462194)
				(3, 0.9981234072289198)
				(4, 0.9985332209680202)
				(5, 0.9987915941793496)
				(6, 0.9990232037293412)
				(7, 0.9991534988751062)
				(8, 0.9992781833000478)
				(9, 0.9993706001012128)
				(10, 0.9994249996817561)
				(11, 0.9994833449904865)
				(12, 0.9995780322951591)
				(13, 0.9995681795776103)
				(14, 0.9996165813942961)
				(15, 0.999646856662305)
				(16, 0.9996374875357916)
				(17, 0.9996908467063693)
				(18, 0.9996922125836588)
				(19, 0.999715866846988)
				(20, 0.9997437819466959)
				
			};
			
			\addplot[
			color=cyan,
			mark=triangle,
			mark size=2pt
			]
			coordinates{(1, 0.8932366664995447)
				(2, 0.9723097989182526)
				(3, 0.970826682231367)
				(4, 0.9826649830913539)
				(5, 0.982498852024745)
				(6, 0.9872683357812961)
				(7, 0.9870694952424083)
				(8, 0.9900621632079168)
				(9, 0.9895785879699214)
				(10, 0.9914583724660571)
				(11, 0.9908869390037283)
				(12, 0.992382728227613)
				(13, 0.9917358307457333)
				(14, 0.9929605908502663)
				(15, 0.9926152256996312)
				(16, 0.9935184496537927)
				(17, 0.9931399809087036)
				(18, 0.9937658231212452)
				(19, 0.9937526466622133)
				(20, 0.99423033891832)
				};
			
			\addplot[
			color=brown,
			mark=square,
			mark size=1.5pt
			]
			coordinates{(1, 0.9999735856561476)
				(2, 0.9999875839422812)
				(3, 0.9999923273376881)
				(4, 0.999994713936848)
				(5, 0.9999961507188789)
				(6, 0.9999971105956681)
				(7, 0.9999977972163957)
				(8, 0.99999831272733)
				(9, 0.9999987140041593)
				(10, 0.9999990352300107)
				(11, 0.9999992981864643)
				(12, 0.9999995174099566)
				(13, 0.9999997029729296)
				(14, 0.9999998620752498)
				(15, 1)
				(16, 0.9999998947136202)
				(17, 0.9999998267467178)
				(18, 0.9999997859713133)
				(19, 0.9999997651681464)
				(20, 0.999999759115476)
				};
			\end{axis}
			\end{tikzpicture}\label{fig:gl_1_Z}}
		\hfil
		\subfloat[{\bf GL-FACT, Metric: \hyperref[metric_Z]{$Z$}}]{\pgfplotsset{width=0.33\textwidth,height=0.18\textwidth,compat=1.9}
			\begin{tikzpicture}
			\begin{axis}[
			xlabel={${k}$},
			ylabel={\hyperref[metric_Z]{$Z$}},
			xmin=0, xmax=20,
			ymin=0.7, ymax=1,
			xtick={0,5,10,15,20},
			ymajorgrids=true,
			grid style=dashed,
			ylabel style={rotate=-90}
			]
			\addplot[
			color=blue,
			mark=o,
			mark size=2pt
			]
			coordinates {(1, 0.9999396256969204)
				(2, 0.9999555170533347)
				(3, 0.999965031208885)
				(4, 0.9999659333612415)
				(5, 0.9999677228119941)
				(6, 0.9999666103006767)
				(7, 0.9999621513197136)
				(8, 0.9999603286033429)
				(9, 0.9999591407931322)
				(10, 0.9999545814276304)
				(11, 0.9999523071900172)
				(12, 0.9999497531414681)
				(13, 0.9999461060285707)
				(14, 0.9999438762258006)
				(15, 1)
				(16, 0.9999996357599408)
				(17, 0.9999989958645984)
				(18, 0.9999979086790814)
				(19, 0.999996565776027)
				(20, 0.9999953633843445)
					
			};
			
			\addplot[
			color=red,
			mark=x,
			mark size=2pt
			]
			coordinates {(1, 0.7176638118132861)
				(2, 0.7455004780267992)
				(3, 0.7626021157613833)
				(4, 0.7752998760732408)
				(5, 0.7846523305677298)
				(6, 0.7923528016642492)
				(7, 0.7997558893232991)
				(8, 0.8063359955453037)
				(9, 0.8121801704433762)
				(10, 0.8173282973430156)
				(11, 0.8219469256963984)
				(12, 0.8259793767145674)
				(13, 0.8300866921412833)
				(14, 0.8338350964631197)
				(15, 0.8371405186639884)
				(16, 0.8404599332144972)
				(17, 0.8434366376180288)
				(18, 0.8463569087210795)
				(19, 0.8490633765040714)
				(20, 0.8514541177377445)
				
			};
			
			\addplot[
			color=black,
			mark=star,
			mark size=2pt
			]
			coordinates {(1, 0.9947192700381087)
				(2, 0.9969978004748151)
				(3, 0.9980751767510534)
				(4, 0.9987094094626768)
				(5, 0.9989278970877387)
				(6, 0.9990428034015294)
				(7, 0.9991834336255927)
				(8, 0.9992964069680403)
				(9, 0.9993864715151106)
				(10, 0.9994437428408329)
				(11, 0.9995101272824426)
				(12, 0.9995660231065734)
				(13, 0.9995683600721268)
				(14, 0.9996200527876683)
				(15, 0.9996259472839225)
				(16, 0.9996282401422408)
				(17, 0.9996686688935053)
				(18, 0.9996755309124871)
				(19, 0.9997217348465037)
				(20, 0.9997279876851688)			
				
			};
			\addplot[
			color=cyan,
			mark=triangle,
			mark size=2pt
			]
			coordinates{(1, 0.7176638118132861)
				(2, 0.9146205655183863)
				(3, 0.878172246707361)
				(4, 0.9246544486357833)
				(5, 0.9057350348185181)
				(6, 0.9310103239972296)
				(7, 0.9187985498399571)
				(8, 0.9351772801865127)
				(9, 0.9264238106548425)
				(10, 0.9387505822108574)
				(11, 0.9319190969074457)
				(12, 0.9412955919914375)
				(13, 0.9360114026494863)
				(14, 0.9437817387877127)
				(15, 0.939621835016416)
				(16, 0.9463623053259648)
				(17, 0.9426495723972481)
				(18, 0.9479549068788827)
				(19, 0.9448907899913421)
				(20, 0.9498333414584168)
				};
			
			\addplot[
			color=brown,
			mark=square,
			mark size=1.5pt
			]
			coordinates{(1, 0.9999735856561476)
				(2, 0.9999875839422812)
				(3, 0.9999923273376881)
				(4, 0.999994713936848)
				(5, 0.9999961507188789)
				(6, 0.9999971105956681)
				(7, 0.9999977972163957)
				(8, 0.99999831272733)
				(9, 0.9999987140041593)
				(10, 0.9999990352300107)
				(11, 0.9999992981864643)
				(12, 0.9999995174099566)
				(13, 0.9999997029729296)
				(14, 0.9999998620752498)
				(15, 1)
				(16, 0.9999998947136202)
				(17, 0.9999998267467178)
				(18, 0.9999997859713133)
				(19, 0.9999997651681464)
				(20, 0.999999759115476)
				};
			
			\end{axis}
			\end{tikzpicture}\label{fig:gl_2_Z}}
		\hfil
		\subfloat[{\bf LF, Metric: \hyperref[metric_Z]{$Z$}}]{\pgfplotsset{width=0.33\textwidth,height=0.18\textwidth,compat=1.9}
			\begin{tikzpicture}
			\begin{axis}[
			xlabel={${k}$},
			ylabel={\hyperref[metric_Z]{$Z$}},
			xmin=0, xmax=20,
			ymin=0.7, ymax=1,
			xtick={0,5,10,15,20},
			ytick={0.7,0.8,0.9,1},
			ymajorgrids=true,
			grid style=dashed,
			ylabel style={rotate=-90}
			]
			\addplot[
			color=blue,
			mark=o,
			mark size=2pt
			]
			coordinates {(1, 0.7526361891413638)
				(2, 0.8108291637299694)
				(3, 0.8419384616588211)
				(4, 0.8618350473417699)
				(5, 0.8764344262352939)
				(6, 0.8878257261458331)
				(7, 0.8970486346021643)
				(8, 0.9043172466256092)
				(9, 0.9106088877893651)
				(10, 0.9968333971479676)
				(11, 0.9921247689776508)
				(12, 0.9881851497427461)
				(13, 0.9850113100956555)
				(14, 0.9823340045172215)
				(15, 0.979996213188703)
				(16, 0.9782425725370162)
				(17, 0.9767969484148185)
				(18, 0.9754809946956532)
				(19, 0.9994823649736437)
				(20, 0.9975310687214431)

			};
			
			\addplot[
			color=red,
			mark=x,
			mark size=2pt
			]
			coordinates {(1, 0.7526361891413638)
				(2, 0.8108291637299694)
				(3, 0.8419384616588211)
				(4, 0.8618350473417699)
				(5, 0.8764344262352939)
				(6, 0.8878257261458331)
				(7, 0.8970486346021643)
				(8, 0.9043172466256092)
				(9, 0.9106088877893651)
				(10, 0.9161626586723886)
				(11, 0.9204578518300951)
				(12, 0.9241974447994065)
				(13, 0.9278236633686459)
				(14, 0.9309864942222632)
				(15, 0.9338319003168735)
				(16, 0.9362368447696351)
				(17, 0.9386270358133529)
				(18, 0.9403671757777169)
				(19, 0.9420074645401821)
				(20, 0.9438539852110053)

			};
			
			\addplot[
			color=black,
			mark=star,
			mark size=2pt
			]
			coordinates {(1, 0.7639266205940035)
				(2, 0.8281952511102977)
				(3, 0.8627445820081104)
				(4, 0.8854985730184916)
				(5, 0.9018125772781138)
				(6, 0.9153293826316845)
				(7, 0.9248134092218625)
				(8, 0.9330190264786918)
				(9, 0.9397462998667486)
				(10, 0.9453238009568177)
				(11, 0.9496457484062013)
				(12, 0.9541552709840722)
				(13, 0.9573933874371768)
				(14, 0.9606261965972633)
				(15, 0.9639587404699339)
				(16, 0.9655452984075262)
				(17, 0.9679013996477931)
				(18, 0.9698272020666501)
				(19, 0.9715483565628751)
				(20, 0.9728081979366235)

			};
			
			\addplot[
			color=cyan,
			mark=triangle,
			mark size=2pt
			]
			coordinates{(1, 0.7526361891413638)
				(2, 0.8214428540473272)
				(3, 0.8525413601551142)
				(4, 0.8788874132982983)
				(5, 0.8917112376938803)
				(6, 0.9063353191188368)
				(7, 0.9132238244804736)
				(8, 0.9234370282007407)
				(9, 0.9279649680634943)
				(10, 0.9358704161837272)
				(11, 0.9389561813188942)
				(12, 0.9439074689819916)
				(13, 0.9464713892828666)
				(14, 0.9506649229990941)
				(15, 0.9529265199137457)
				(16, 0.9564742688313292)
				(17, 0.9573034952622381)
				(18, 0.9600813058332126)
				(19, 0.9615737108380881)
				(20, 0.9635176272815913)
				};
			
			\addplot[
			color=brown,
			mark=square,
			mark size=1.5pt
			]
			coordinates{(1, 0.7717118232210773)
				(2, 0.8426041021120017)
				(3, 0.8840734268539882)
				(4, 0.9134963810029252)
				(5, 0.9363185972884983)
				(6, 0.9549657057448282)
				(7, 0.9707316113013115)
				(8, 0.9843886598937087)
				(9, 0.9964350304865162)
				(10, 0.9975587350370257)
				(11, 0.9952946353148083)
				(12, 0.9941837740526697)
				(13, 0.9939001904487457)
				(14, 0.9942196300097231)
				(15, 0.9949839208989543)
				(16, 0.9960791414020759)
				(17, 0.9974217721013414)
				(18, 0.9989496497283519)
				(19, 0.9995976355651346)
				(20, 0.9987416652495145)
				};
			
			\end{axis}
			\end{tikzpicture}\label{fig:lf_Z}}
		\hfil
		\subfloat[{\bf GL-CUSTOM, Metric: \hyperref[metric_L]{$L$}}]{\pgfplotsset{width=0.33\textwidth,height=0.18\textwidth,compat=1.9}
			\begin{tikzpicture}
			\begin{axis}[
			xlabel={${k}$},
			ylabel={\hyperref[metric_L]{$L$}},
			xmin=0, xmax=20,
			ymin=0, ymax=0.2,
			xtick={0,5,10,15,20},
			ytick={0,0.1,0.2},
			ymajorgrids=true,
			grid style=dashed,
			ylabel style={rotate=-90},
			]
			\addplot[
			color=blue,
			mark=o,
			mark size=2pt
			]
			coordinates {(1, 0.15063780710653246)
				(2, 0.13358364869330705)
				(3, 0.12871858946884923)
				(4, 0.1253107725112325)
				(5, 0.1222404502048407)
				(6, 0.11868824369637608)
				(7, 0.11657571256104575)
				(8, 0.11369339269148934)
				(9, 0.11386827292125684)
				(10, 0.11468892962868707)
				(11, 0.11634239609946556)
				(12, 0.11650353136024133)
				(13, 0.11631025778716965)
				(14, 0.11588266702456562)
				(15, 0.11781313565584825)
				(16, 0.11862253664194923)
				(17, 0.1177316937654791)
				(18, 0.11628274549231145)
				(19, 0.1154079599095793)
				(20, 0.11420550431363932)

			};
			
			\addplot[
			color=red,
			mark=x,
			mark size=2pt
			]
			coordinates {(1, 0)
				(2, 0)
				(3, 0)
				(4, 0)
				(5, 0)
				(6, 0)
				(7, 0)
				(8, 0)
				(9, 0)
				(10, 0)
				(11, 0)
				(12, 0)
				(13, 0)
				(14, 0)
				(15, 0)
				(16, 0)
				(17, 0)
				(18, 0)
				(19, 0)
				(20, 0)
								
			};
			
			\addplot[
			color=black,
			mark=star,
			mark size=2pt
			]
			coordinates {(1, 0.16695000144987354)
				(2, 0.14054804615821878)
				(3, 0.13468944380070239)
				(4, 0.1329746461316798)
				(5, 0.1310467557854786)
				(6, 0.12681274237081108)
				(7, 0.12277404602213084)
				(8, 0.12085626716834882)
				(9, 0.11848536611348301)
				(10, 0.1191238201480736)
				(11, 0.12006752811278204)
				(12, 0.11975458445187187)
				(13, 0.11920188694609973)
				(14, 0.11861176792041096)
				(15, 0.12059009396391591)
				(16, 0.12233734860134171)
				(17, 0.12001730255816244)
				(18, 0.11801427312768357)
				(19, 0.11651887695185992)
				(20, 0.11589057889836077)

			};
			\addplot[
			color=cyan,
			mark=triangle,
			mark size=2pt
			]
			coordinates{(1, 0.0)
				(2, 0.13174919478943303)
				(3, 0.11212666080236035)
				(4, 0.11633447200247393)
				(5, 0.10419826032313874)
				(6, 0.10301024686608246)
				(7, 0.0944586516231857)
				(8, 0.09806686609069216)
				(9, 0.09258326260006153)
				(10, 0.095148419035753)
				(11, 0.08745635097570739)
				(12, 0.09092889837670186)
				(13, 0.08521009621049765)
				(14, 0.09020981096268477)
				(15, 0.08643425728064809)
				(16, 0.08922934134544222)
				(17, 0.0878111048730379)
				(18, 0.08972106105255674)
				(19, 0.08575019329617667)
				(20, 0.08632857292630775)
				};
			
			\addplot[
			color=brown,
			mark=square,
			mark size=1.5pt
			]
			coordinates{(1, 0.15159185095513575)
				(2, 0.134695392256766)
				(3, 0.12949875523422402)
				(4, 0.12655894931901973)
				(5, 0.1234695732830605)
				(6, 0.11976863806904799)
				(7, 0.1178777271226391)
				(8, 0.1150800389874694)
				(9, 0.11529476649169731)
				(10, 0.11596938454916962)
				(11, 0.11757163190452806)
				(12, 0.11794015342763986)
				(13, 0.11770574947710379)
				(14, 0.11723891123678556)
				(15, 0.11781313565584825)
				(16, 0.11872528614021614)
				(17, 0.11789836296390656)
				(18, 0.1165450282681505)
				(19, 0.1157293554679767)
				(20, 0.1145785215573365)
				};
			
			\end{axis}
			\end{tikzpicture}\label{fig:gl_1_L}}
		\hfil
		\subfloat[{\bf GL-FACT, Metric: \hyperref[metric_L]{$L$}}]{\pgfplotsset{width=0.33\textwidth,height=0.18\textwidth,compat=1.9}
			\begin{tikzpicture}
			\begin{axis}[
			xlabel={${k}$},
			ylabel={\hyperref[metric_L]{$L$}},
			xmin=0, xmax=20,
			ymin=0, ymax=0.2,
			xtick={0,5,10,15,20},
			ytick={0,0.1,0.2},
			ymajorgrids=true,
			grid style=dashed,
			ylabel style={rotate=-90},
			]
			\addplot[
			color=blue,
			mark=o,
			mark size=2pt
			]
			coordinates {(1, 0.091729875638232)
				(2, 0.09758047567449074)
				(3, 0.1010643557199388)
				(4, 0.1048405587645209)
				(5, 0.10776499547207413)
				(6, 0.11013299022667537)
				(7, 0.11240628595774407)
				(8, 0.11435475122902035)
				(9, 0.11590841237423416)
				(10, 0.11741732220441725)
				(11, 0.11852532551278633)
				(12, 0.11990570185411359)
				(13, 0.12096980929792109)
				(14, 0.12258803662529964)
				(15, 0.12361816347669605)
				(16, 0.12426885681486527)
				(17, 0.1253976746204187)
				(18, 0.12632563941104036)
				(19, 0.12711679249422894)
				(20, 0.12787155355483099)

			};
			
			\addplot[
			color=red,
			mark=x,
			mark size=2pt
			]
			coordinates {(1, 0)
				(2, 0)
				(3, 0)
				(4, 0)
				(5, 0)
				(6, 0)
				(7, 0)
				(8, 0)
				(9, 0)
				(10, 0)
				(11, 0)
				(12, 0)
				(13, 0)
				(14, 0)
				(15, 0)
				(16, 0)
				(17, 0)
				(18, 0)
				(19, 0)
				(20, 0)
								
			};
		
			\addplot[
			color=black,
			mark=star,
			mark size=2pt
			]
			coordinates {(1, 0.09516648814875438)
				(2, 0.1012635130674154)
				(3, 0.1053452769913435)
				(4, 0.10797076396625027)
				(5, 0.1097954442268999)
				(6, 0.11174979912799303)
				(7, 0.11332531711377004)
				(8, 0.11544438040435741)
				(9, 0.11667709695569659)
				(10, 0.1171019292763498)
				(11, 0.11939028147214557)
				(12, 0.12032900506627806)
				(13, 0.12191004194885334)
				(14, 0.12215510801618752)
				(15, 0.12417236136511105)
				(16, 0.12489840314784689)
				(17, 0.12482082310443984)
				(18, 0.12718904258980665)
				(19, 0.12676400115358616)
				(20, 0.12800440737367835)

			};
			\addplot[
			color=cyan,
			mark=triangle,
			mark size=2pt
			]
			coordinates{(1, 0.0)
				(2, 0.059870460166441854)
				(3, 0.04120510115823441)
				(4, 0.05962674818060364)
				(5, 0.05166441638692602)
				(6, 0.06428035894170778)
				(7, 0.055316943041019404)
				(8, 0.06485004920204032)
				(9, 0.05955792594574708)
				(10, 0.0660650919722464)
				(11, 0.062158777225926214)
				(12, 0.06715170432767836)
				(13, 0.06383641557318105)
				(14, 0.06853310409345548)
				(15, 0.064292453297279)
				(16, 0.06972229676356045)
				(17, 0.06562465931062567)
				(18, 0.07062141344251469)
				(19, 0.06700564102666685)
				(20, 0.0713988887130878)
				
				};
			
			\addplot[
			color=brown,
			mark=square,
			mark size=1.5pt
			]
			coordinates{(1, 0.09197996139251075)
				(2, 0.09793467773633713)
				(3, 0.10146253489789028)
				(4, 0.1052396573107628)
				(5, 0.10817940034003216)
				(6, 0.11050539415298133)
				(7, 0.11267242777822876)
				(8, 0.11456931000035189)
				(9, 0.11610839058146254)
				(10, 0.11754847197512311)
				(11, 0.11860090042100091)
				(12, 0.11996510923137564)
				(13, 0.12101368304486385)
				(14, 0.12264514485462995)
				(15, 0.12361816347669605)
				(16, 0.1242730436526083)
				(17, 0.1254115582884779)
				(18, 0.12633238566978405)
				(19, 0.1271336874366622)
				(20, 0.12789023105284825)
				};
			
			\end{axis}
			\end{tikzpicture}\label{fig:gl_2_L}}
		\hfil
		\subfloat[{\bf LF, Metric: \hyperref[metric_L]{$L$}}]{\pgfplotsset{width=0.33\textwidth,height=0.18\textwidth,compat=1.9}
			\begin{tikzpicture}
			\begin{axis}[
			xlabel={${k}$},
			ylabel={\hyperref[metric_L]{$L$}},
			xmin=0, xmax=20,
			ymin=0, ymax=0.3,
			xtick={0,5,10,15,20},
			ytick={0,0.1,0.2,0.3},
			ymajorgrids=true,
			grid style=dashed,
			ylabel style={rotate=-90},
			]
			\addplot[
			color=blue,
			mark=o,
			mark size=2pt
			]
			coordinates {(1, 0.0)
				(2, 0.0)
				(3, 0.0)
				(4, 0.0)
				(5, 0.0)
				(6, 0.0)
				(7, 0.0)
				(8, 0.0)
				(9, 0.0)
				(10, 0.1460956405021171)
				(11, 0.1358410754997925)
				(12, 0.12665641516702744)
				(13, 0.11575536273162881)
				(14, 0.10743678914476311)
				(15, 0.1004601782970098)
				(16, 0.09388104384395209)
				(17, 0.08715112234605377)
				(18, 0.08252987650194095)
				(19, 0.15419832253135005)
				(20, 0.14851570211706927)

			};
			\addplot[
			color=red,
			mark=x,
			mark size=2pt
			]
			coordinates {(1, 0)
				(2, 0)
				(3, 0)
				(4, 0)
				(5, 0)
				(6, 0)
				(7, 0)
				(8, 0)
				(9, 0)
				(10, 0)
				(11, 0)
				(12, 0)
				(13, 0)
				(14, 0)
				(15, 0)
				(16, 0)
				(17, 0)
				(18, 0)
				(19, 0)
				(20, 0)
								
			};
			
			\addplot[
			color=black,
			mark=star,
			mark size=2pt
			]
			coordinates {(1, 0.0855263157894737)
				(2, 0.13961358136721114)
				(3, 0.17948852465214754)
				(4, 0.2035390199637022)
				(5, 0.22385725898798692)
				(6, 0.23607561100046312)
				(7, 0.24642692440021094)
				(8, 0.25309513456774124)
				(9, 0.2577815542808172)
				(10, 0.25891711588597954)
				(11, 0.2628699246156143)
				(12, 0.25601338338148777)
				(13, 0.25346050401325193)
				(14, 0.2537623149876443)
				(15, 0.2522769146778486)
				(16, 0.25088792638920243)
				(17, 0.24822556525275097)
				(18, 0.24632146505073915)
				(19, 0.2436322896554283)
				(20, 0.24130118054055735)

			};
			\addplot[
			color=cyan,
			mark=triangle,
			mark size=2pt
			]
			coordinates{(1, 0.0)
				(2, 0.07915032732693807)
				(3, 0.07276697940973118)
				(4, 0.1217105263157894)
				(5, 0.11606886288998348)
				(6, 0.1527298073095487)
				(7, 0.14398852039647209)
				(8, 0.16988380818490884)
				(9, 0.16273390462506868)
				(10, 0.17937655741364925)
				(11, 0.16964832003610955)
				(12, 0.18276512164913641)
				(13, 0.1777748390785994)
				(14, 0.18581999141209696)
				(15, 0.17989401061807136)
				(16, 0.18595584862323883)
				(17, 0.1782561141624036)
				(18, 0.18391787082219657)
				(19, 0.17828717782106934)
				(20, 0.18091008791848934)
				};
				
			\addplot[
			color=brown,
			mark=square,
			mark size=1.5pt
			]
			coordinates{(1, 0.08735065033272837)
				(2, 0.1456943814277072)
				(3, 0.18027261040532366)
				(4, 0.1999213191023536)
				(5, 0.20677198113819034)
				(6, 0.20432607241392378)
				(7, 0.19425314158895154)
				(8, 0.18016768549181297)
				(9, 0.1604558495736468)
				(10, 0.1615720335321051)
				(11, 0.17011831779583458)
				(12, 0.17629953853773048)
				(13, 0.1782605020962748)
				(14, 0.17698706865704303)
				(15, 0.175619241561999)
				(16, 0.17228299561910174)
				(17, 0.16756789594905552)
				(18, 0.16150736589669298)
				(19, 0.15889830362074867)
				(20, 0.16368359739963642)
				};
			
			\end{axis}
			\end{tikzpicture}\label{fig:lf_L}}
		\vfil
		\subfloat{\pgfplotsset{width=.5\textwidth,compat=1.9}
			\begin{tikzpicture}
			\begin{customlegend}[legend entries={{FairRec},{Top-$k$},{Random-$k$},{Mixed-$k$},{PR-$k$}},legend columns=5,legend style={/tikz/every even column/.append style={column sep=0.5cm}}]
			\addlegendimage{blue,mark=o,sharp plot}
			\addlegendimage{red,mark=x,sharp plot}
			\addlegendimage{black,mark=star,sharp plot}
			\addlegendimage{cyan,mark=triangle,sharp plot}
			\addlegendimage{brown,mark=square,sharp plot}
			\end{customlegend}
			\end{tikzpicture}}
	}\caption{Producer-Side Performances with MMS Guarantee. First row: fraction of satisfied producers (\hyperref[metric_H]{$H$}). Second row: inequality in producer exposures (\hyperref[metric_Z]{$Z$}). Third row: exposure loss on producers (\hyperref[metric_L]{$L$}).}\label{fig:producer_side_mms}
\end{figure*}
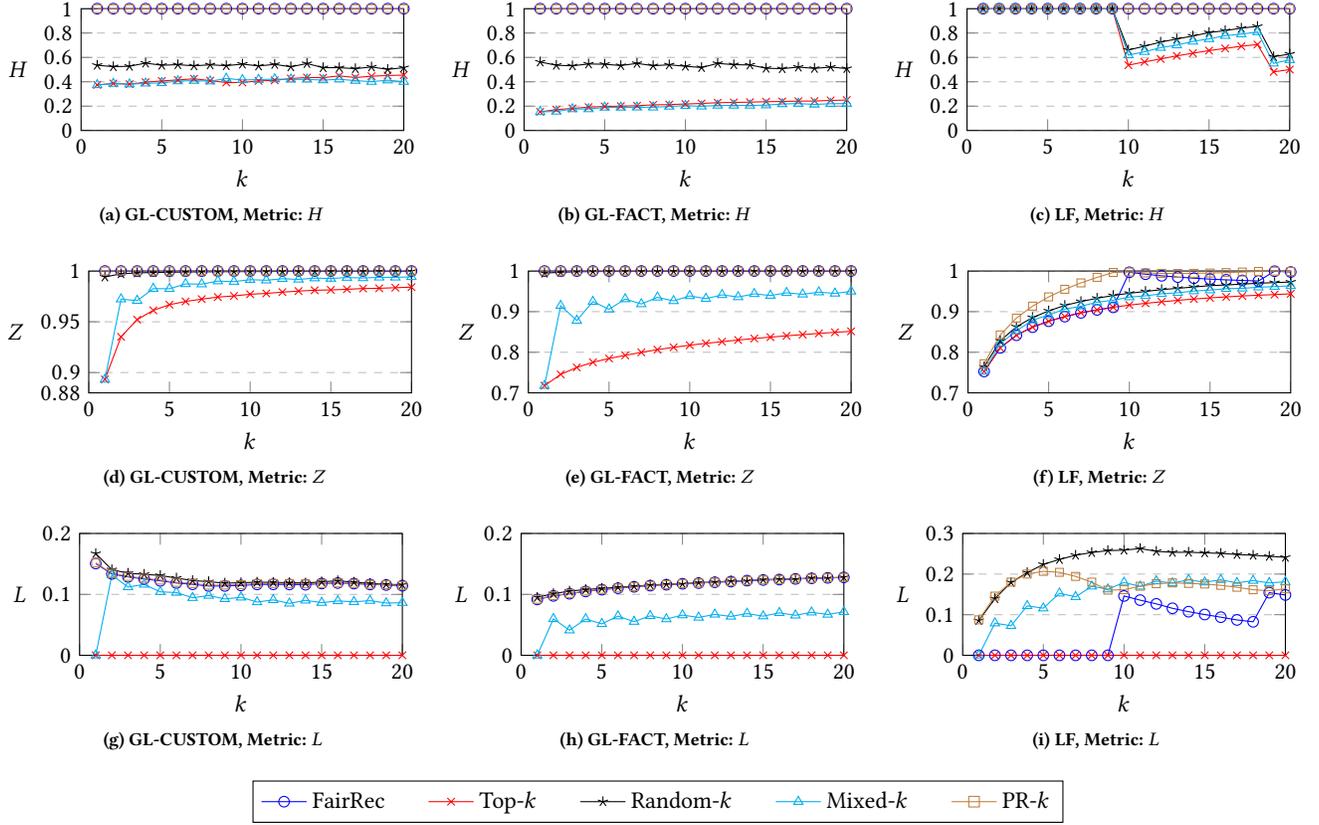
\begin{figure*}[t!]
	\center{
		
		\subfloat[{GL-CUSTOM, Metric: \hyperref[metric_Y]{$Y$}}]{\pgfplotsset{width=0.29\textwidth,height=0.19\textwidth,compat=1.9}
			\begin{tikzpicture}
			\begin{axis}[
			xlabel={${k}$},
			ylabel={\hyperref[metric_Y]{$Y$}},
			xmin=0, xmax=20,
			ymin=0, ymax=0.023,
			xtick={0,5,10,15,20},
			ytick={0,0.01,0.02},
			ymajorgrids=true,
			grid style=dashed,
			ylabel style={rotate=-90},
			yticklabel style={
				/pgf/number format/fixed,
				/pgf/number format/precision=2},
			]
			\addplot[
			color=blue,
			mark=o,
			mark size=2pt
			]
			coordinates {(1, 0.0017043940014530947)
				(2, 0.0004297908599545868)
				(3, 0.00025254504080081657)
				(4, 0.0001825527827971225)
				(5, 0.00014599369289113522)
				(6, 0.000123931801333511)
				(7, 0.00010423651068556081)
				(8, 8.73786845939502e-05)
				(9, 7.651386569669561e-05)
				(10, 6.868359034227222e-05)
				(11, 5.4298185598470845e-05)
				(12, 5.1093906221809745e-05)
				(13, 4.9694403317882844e-05)
				(14, 4.617038263366066e-05)
				(15, 4.3194117053634275e-05)
				(16, 3.6148393335023115e-05)
				(17, 3.577481581319694e-05)
				(18, 3.224245668502598e-05)
				(19, 2.911229499068566e-05)
				(20, 2.66331347177769e-05)
				
			};
			
			\addplot[
			color=red,
			mark=x,
			mark size=2pt
			]
			coordinates {(1, 0.0)
				(2, 0.0)
				(3, 0.0)
				(4, 0.0)
				(5, 0.0)
				(6, 0.0)
				(7, 0.0)
				(8, 0.0)
				(9, 0.0)
				(10, 0.0)
				(11, 0.0)
				(12, 0.0)
				(13, 0.0)
				(14, 0.0)
				(15, 0.0)
				(16, 0.0)
				(17, 0.0)
				(18, 0.0)
				(19, 0.0)
				(20, 0.0)				
			};
			
			\addplot[
			color=black,
			mark=star,
			mark size=2pt
			]
			coordinates {(1, 0.01793699896262849)
				(2, 0.017254360420913325)
				(3, 0.017555637942408508)
				(4, 0.01762009723599878)
				(5, 0.018024716334417018)
				(6, 0.0180258934413201)
				(7, 0.018373973705397974)
				(8, 0.018495244578696536)
				(9, 0.018723102951333316)
				(10, 0.01904744923170269)
				(11, 0.01911070273947384)
				(12, 0.019426219003492465)
				(13, 0.019583444149473488)
				(14, 0.019790312852975226)
				(15, 0.01967163244027707)
				(16, 0.020048206923068662)
				(17, 0.02011554289027105)
				(18, 0.02028585942780096)
				(19, 0.020154737861342982)
				(20, 0.020429666977341727)				
			};
			
			\addplot[
			color=cyan,
			mark=triangle,
			mark size=2pt
			]
			coordinates{(1, 0.0)
				(2, 5.797393867828919e-05)
				(3, 2.58277929860342e-05)
				(4, 3.232081525205657e-05)
				(5, 2.3524756471809378e-05)
				(6, 2.957484523607126e-05)
				(7, 2.37848898667461e-05)
				(8, 2.5690638348711794e-05)
				(9, 2.3927822577129187e-05)
				(10, 2.6115789913302553e-05)
				(11, 2.2800709029431594e-05)
				(12, 2.494930491211597e-05)
				(13, 2.2260522376273897e-05)
				(14, 2.397586803616847e-05)
				(15, 2.1986767326018676e-05)
				(16, 2.3490245280431163e-05)
				(17, 2.143313070098517e-05)
				(18, 2.1439594575261448e-05)
				(19, 2.1594589020488627e-05)
				(20, 2.1345159157229842e-05)
				};
			
			\addplot[
			color=brown,
			mark=square,
			mark size=1.5pt
			]
			coordinates{(1, 0.017816278049679612)
				(2, 0.017401453165744835)
				(3, 0.017505818211718743)
				(4, 0.017661277361124224)
				(5, 0.01787209474824784)
				(6, 0.018096526407296993)
				(7, 0.018289086795492515)
				(8, 0.018536350473349138)
				(9, 0.018788924335213404)
				(10, 0.019045362294775055)
				(11, 0.019282657325181925)
				(12, 0.019446693626609245)
				(13, 0.019615076404784312)
				(14, 0.019826590263129662)
				(15, 0.020001611027667347)
				(16, 0.0202654823596313)
				(17, 0.02050499102306031)
				(18, 0.020679517328389463)
				(19, 0.02087529215519208)
				(20, 0.021057898296705683)
				};
			
			\end{axis}
			\end{tikzpicture}\label{fig:gl_1_Y}}
		\hfil
		\subfloat[{GL-FACT, Metric: \hyperref[metric_Y]{$Y$}}]{\pgfplotsset{width=0.29\textwidth,height=0.19\textwidth,compat=1.9}
			\begin{tikzpicture}
			\begin{axis}[
			xlabel={${k}$},
			ylabel={\hyperref[metric_Y]{$Y$}},
			xmin=0, xmax=20,
			ymin=0, ymax=0.07,
			xtick={0,5,10,15,20},
			ytick={0,0.03,0.06},
			ymajorgrids=true,
			grid style=dashed,
			ylabel style={rotate=-90},
			]
			\addplot[
			color=blue,
			mark=o,
			mark size=2pt
			]
			coordinates {(1, 0.003461491025148509)
				(2, 0.0004604557876897308)
				(3, 7.442243192437262e-05)
				(4, 1.593056865385296e-05)
				(5, 3.3261057540291758e-06)
				(6, 7.034942156005791e-07)
				(7, 1.9037197251375726e-07)
				(8, 4.190200458767048e-08)
				(9, 1.392809104192881e-08)
				(10, 3.2247626186159614e-09)
				(11, 1.1243835871008896e-09)
				(12, 4.918355876403647e-11)
				(13, 1.386193467962089e-10)
				(14, 7.983377295901377e-11)
				(15, 1.4499045992396074e-10)
				(16, 0.0)
				(17, 4.875988014100891e-12)
				(18, 0.0)
				(19, 0.0)
				(20, 3.795349979216101e-11)

			};
			
			\addplot[
			color=red,
			mark=x,
			mark size=2pt
			]
			coordinates {(1, 0.0)
				(2, 0.0)
				(3, 0.0)
				(4, 0.0)
				(5, 0.0)
				(6, 0.0)
				(7, 0.0)
				(8, 0.0)
				(9, 0.0)
				(10, 0.0)
				(11, 0.0)
				(12, 0.0)
				(13, 0.0)
				(14, 0.0)
				(15, 0.0)
				(16, 0.0)
				(17, 0.0)
				(18, 0.0)
				(19, 0.0)
				(20, 0.0)

			};
			
			\addplot[
			color=black,
			mark=star,
			mark size=2pt
			]
			coordinates {(1, 0.06208407394457786)
				(2, 0.04483476168066226)
				(3, 0.0369876908177635)
				(4, 0.032228699659135364)
				(5, 0.028443570769447292)
				(6, 0.026509695254301992)
				(7, 0.024582168494921806)
				(8, 0.02308695152514603)
				(9, 0.02198969746499456)
				(10, 0.021136268765150977)
				(11, 0.019800020849761266)
				(12, 0.019173360697236255)
				(13, 0.01814813876300008)
				(14, 0.017859807658279395)
				(15, 0.01739574757102449)
				(16, 0.016911371319659894)
				(17, 0.016417099037098728)
				(18, 0.015848580556103855)
				(19, 0.015466337147500326)
				(20, 0.01501734530281081)

			};
			
			\addplot[
			color=cyan,
			mark=triangle,
			mark size=2pt
			]
			coordinates{(1, 0.0)
				(2, 0.005260719185561497)
				(3, 0.000663468060975349)
				(4, 0.0016351180919513993)
				(5, 0.00036337604923014036)
				(6, 0.0006746135278466422)
				(7, 0.00019905628694489772)
				(8, 0.0003388454988461903)
				(9, 0.00011516580029139897)
				(10, 0.0001862415217199035)
				(11, 7.464056665330385e-05)
				(12, 0.00010264565421183724)
				(13, 4.229473262612029e-05)
				(14, 6.278640142952336e-05)
				(15, 3.15944992137817e-05)
				(16, 4.0642192177452616e-05)
				(17, 2.189681049984065e-05)
				(18, 2.96788449152787e-05)
				(19, 1.5709569474629235e-05)
				(20, 1.955382183506999e-05)
				};
			
			\addplot[
			color=brown,
			mark=square,
			mark size=1.5pt
			]
			coordinates{(1, 0.06197430277316864)
				(2, 0.045071412268355295)
				(3, 0.03762724837587231)
				(4, 0.03272594012506815)
				(5, 0.02959544701423821)
				(6, 0.02709406058620529)
				(7, 0.025177912945710083)
				(8, 0.0236744832911592)
				(9, 0.0224496444748936)
				(10, 0.02137457309541077)
				(11, 0.020496352077064517)
				(12, 0.01960540484419265)
				(13, 0.018796746468934568)
				(14, 0.018195959575431533)
				(15, 0.017520676247796678)
				(16, 0.016924577803115428)
				(17, 0.016370562514175197)
				(18, 0.015868420480000468)
				(19, 0.015422182900300683)
				(20, 0.01497378131115677)
				};
			
			\end{axis}
			\end{tikzpicture}\label{fig:gl_2_Y}}
		\hfil
		\subfloat[{LF, Metric: \hyperref[metric_Y]{$Y$}}]{\pgfplotsset{width=0.29\textwidth,height=0.19\textwidth,compat=1.9}
			\begin{tikzpicture}
			\begin{axis}[
			xlabel={${k}$},
			ylabel={\hyperref[metric_Y]{$Y$}},
			xmin=0, xmax=20,
			ymin=0, ymax=0.13,
			xtick={0,5,10,15,20},
			ytick={0,0.1,0.1,0.25},
			ymajorgrids=true,
			grid style=dashed,
			ylabel style={rotate=-90}
			]
			\addplot[
			color=blue,
			mark=o,
			mark size=2pt
			]
			coordinates {(1, 0.0)
				(2, 0.0)
				(3, 0.0)
				(4, 0.0)
				(5, 0.0)
				(6, 0.0)
				(7, 0.0)
				(8, 0.0)
				(9, 0.0)
				(10, 3.599168422628413e-08)
				(11, 2.513892028644276e-08)
				(12, 1.2606466906688133e-08)
				(13, 8.35989002801812e-09)
				(14, 4.549814186626957e-09)
				(15, 2.327429238402183e-09)
				(16, 2.6852957436063394e-09)
				(17, 1.357642074280158e-09)
				(18, 1.4500262299793974e-09)
				(19, 3.7473965270429564e-09)
				(20, 2.778008744513805e-09)
				
			};
			
			\addplot[
			color=red,
			mark=x,
			mark size=2pt
			]
			coordinates {(1, 0.0)
				(2, 0.0)
				(3, 0.0)
				(4, 0.0)
				(5, 0.0)
				(6, 0.0)
				(7, 0.0)
				(8, 0.0)
				(9, 0.0)
				(10, 0.0)
				(11, 0.0)
				(12, 0.0)
				(13, 0.0)
				(14, 0.0)
				(15, 0.0)
				(16, 0.0)
				(17, 0.0)
				(18, 0.0)
				(19, 0.0)
				(20, 0.0)
				
			};
			
			\addplot[
			color=black,
			mark=star,
			mark size=2pt
			]
			coordinates {(1, 0.12015407423780398)
				(2, 0.08794479528368995)
				(3, 0.07273599768331776)
				(4, 0.06344075001618305)
				(5, 0.057195718790538454)
				(6, 0.05286012113244257)
				(7, 0.05043555751941485)
				(8, 0.045956686672187365)
				(9, 0.04359215538343012)
				(10, 0.0421769279574879)
				(11, 0.040984821272136755)
				(12, 0.03821407532110564)
				(13, 0.03700157727007322)
				(14, 0.03761528663262455)
				(15, 0.03544189615057377)
				(16, 0.03474507704585105)
				(17, 0.03370604009097508)
				(18, 0.03302600332322024)
				(19, 0.032167162874875005)
				(20, 0.03125736101681002)
				
			};
			\addplot[
			color=cyan,
			mark=triangle,
			mark size=2pt
			]
			coordinates{(1, 0.0)
				(2, 0.0013154836952041762)
				(3, 3.750397273089848e-05)
				(4, 0.00012219946899259105)
				(5, 7.606507085307876e-06)
				(6, 3.66111683027565e-05)
				(7, 2.213140572048084e-06)
				(8, 1.803108679078536e-05)
				(9, 1.478133005059944e-06)
				(10, 1.3075831291725267e-06)
				(11, 4.615580219633277e-07)
				(12, 7.176361933710428e-06)
				(13, 4.786804122996531e-07)
				(14, 6.042857779922118e-07)
				(15, 7.353508032648256e-08)
				(16, 5.086592564191513e-07)
				(17, 2.818249290535493e-07)
				(18, 3.453971375204488e-07)
				(19, 8.792015843225165e-08)
				(20, 5.068572343991874e-07)
				};
			
			\addplot[
			color=brown,
			mark=square,
			mark size=1.5pt
			]
			coordinates{(1, 0.11940137379538748)
				(2, 0.08635847250362083)
				(3, 0.07112673229704121)
				(4, 0.062279842521648476)
				(5, 0.05672724801493963)
				(6, 0.052031035145926836)
				(7, 0.04797523566625057)
				(8, 0.04549015961645075)
				(9, 0.043162676560672235)
				(10, 0.04208051612034533)
				(11, 0.041002318350758384)
				(12, 0.03937083476714338)
				(13, 0.038287317134323005)
				(14, 0.037060918034199815)
				(15, 0.03566034114020945)
				(16, 0.03455047871560757)
				(17, 0.03376833519347011)
				(18, 0.03279569175486361)
				(19, 0.031909465703503684)
				(20, 0.0313953257346975)
				};
			
			\end{axis}
			\end{tikzpicture}\label{fig:lf_Y}}
		\vfil
		\subfloat[{GL-CUSTOM, Metric: \hyperref[sec:cust_metric]{${\mu_{\phi}}$}}]{\pgfplotsset{width=0.29\textwidth,height=0.19\textwidth,compat=1.9}
			\begin{tikzpicture}
			\begin{axis}[
			xlabel={${k}$},
			ylabel={\hyperref[sec:cust_metric]{${\mu_{\phi}}$}},
			xmin=0, xmax=20,
			ymin=0, ymax=1,
			xtick={0,5,10,15,20},
			ytick={0,0.2,0.4,0.6,0.8,1},
			ymajorgrids=true,
			grid style=dashed,
			ylabel style={rotate=-90},
			]
			\addplot[
			color=blue,
			mark=o,
			mark size=2pt
			]
			coordinates {(1, 0.6687787903910435)
				(2, 0.777484263316787)
				(3, 0.8253076920558161)
				(4, 0.8512719350117013)
				(5, 0.8695206725262395)
				(6, 0.8837216104886956)
				(7, 0.894536161443528)
				(8, 0.9024575236049704)
				(9, 0.908767412225598)
				(10, 0.9142709794009495)
				(11, 0.9191403693808834)
				(12, 0.9232406115191909)
				(13, 0.9272782909709504)
				(14, 0.9309744220347547)
				(15, 0.9330893320217132)
				(16, 0.9363527740026186)
				(17, 0.9392591914577787)
				(18, 0.9418653297603778)
				(19, 0.9443201494634443)
				(20, 0.9466278029121776)			
				
			};
			
			\addplot[
			color=red,
			mark=x,
			mark size=2pt
			]
			coordinates {(1, 1.0)
				(2, 1.0)
				(3, 1.0)
				(4, 1.0)
				(5, 1.0)
				(6, 1.0)
				(7, 1.0)
				(8, 1.0)
				(9, 1.0)
				(10, 1.0)
				(11, 1.0)
				(12, 1.0)
				(13, 1.0)
				(14, 1.0)
				(15, 1.0)
				(16, 1.0)
				(17, 1.0)
				(18, 1.0)
				(19, 1.0)
				(20, 1.0)
								
			};
			
			\addplot[
			color=black,
			mark=star,
			mark size=2pt
			]
			coordinates {(1, 0.0412797818053765)
				(2, 0.04999715698535344)
				(3, 0.05547123493683023)
				(4, 0.061423130293976703)
				(5, 0.06562068589494023)
				(6, 0.0718180848194978)
				(7, 0.07543312533217843)
				(8, 0.0809899851145582)
				(9, 0.08421441070539905)
				(10, 0.08780541838419469)
				(11, 0.09153270343805849)
				(12, 0.09576016767225895)
				(13, 0.09841902964165213)
				(14, 0.10212345968679283)
				(15, 0.10586405912621544)
				(16, 0.10817840578202975)
				(17, 0.11189178669094524)
				(18, 0.11540474009587612)
				(19, 0.11857517367827243)
				(20, 0.1214133762645215)
											
			};
			\addplot[
			color=cyan,
			mark=triangle,
			mark size=2pt
			]
			coordinates{(1, 1.0)
				(2, 0.6546582332944494)
				(3, 0.8006455659051205)
				(4, 0.69488570650355)
				(5, 0.7740083210937582)
				(6, 0.7135168150434218)
				(7, 0.768039257196075)
				(8, 0.725876222763723)
				(9, 0.7667112901021682)
				(10, 0.7351316180860856)
				(11, 0.7675998065956751)
				(12, 0.7422029333936575)
				(13, 0.7693594896013705)
				(14, 0.7482440901350443)
				(15, 0.7719102268259541)
				(16, 0.7540248938461448)
				(17, 0.7745715629960214)
				(18, 0.7590335044137733)
				(19, 0.7771857027420206)
				(20, 0.7638514966746964)
				};
			
			\addplot[
			color=brown,
			mark=square,
			mark size=1.5pt
			]
			coordinates{(1, 0.042326772889362624)
				(2, 0.04957176500823769)
				(3, 0.05565030481664764)
				(4, 0.0614841601791565)
				(5, 0.06696139194142138)
				(6, 0.07154908554270696)
				(7, 0.07595623827060194)
				(8, 0.08038803730806111)
				(9, 0.08442641378212044)
				(10, 0.08826079494644398)
				(11, 0.09186642084611615)
				(12, 0.09563937595334007)
				(13, 0.09923596775333826)
				(14, 0.10229962431524756)
				(15, 0.10541143885215962)
				(16, 0.10852281511052626)
				(17, 0.1116980372017332)
				(18, 0.11490889795153103)
				(19, 0.11777921296711251)
				(20, 0.12069856913620906)
				};
			
			\end{axis}
			\end{tikzpicture}\label{fig:gl_1_mu}}
		\hfil
		\subfloat[{GL-FACT, Metric: \hyperref[sec:cust_metric]{${\mu_{\phi}}$}}]{\pgfplotsset{width=0.29\textwidth,height=0.19\textwidth,compat=1.9}
			\begin{tikzpicture}
			\begin{axis}[
			xlabel={${k}$},
			ylabel={\hyperref[sec:cust_metric]{${\mu_{\phi}}$}},
			xmin=0, xmax=20,
			ymin=0, ymax=1,
			xtick={0,5,10,15,20},
			ytick={0,0.2,0.4,0.6,0.8,1},
			ymajorgrids=true,
			grid style=dashed,
			ylabel style={rotate=-90},
			]
			\addplot[
			color=blue,
			mark=o,
			mark size=2pt
			]
			coordinates {(1, 0.8993218695807429)
				(2, 0.9121510738185273)
				(3, 0.9191688372804868)
				(4, 0.9239689041785398)
				(5, 0.9274160391177372)
				(6, 0.9300680249241254)
				(7, 0.9321804261714608)
				(8, 0.9339206807243935)
				(9, 0.9353770668022859)
				(10, 0.9366625837989295)
				(11, 0.9378077152017159)
				(12, 0.938824297002648)
				(13, 0.9397720230722668)
				(14, 0.9405933878107977)
				(15, 0.9406267120727122)
				(16, 0.9413916353200327)
				(17, 0.9420992178071899)
				(18, 0.9427666869138502)
				(19, 0.9433793329755414)
				(20, 0.9439498822796131)
												
			};
			
			\addplot[
			color=red,
			mark=x,
			mark size=2pt
			]
			coordinates {(1, 1.0)
				(2, 1.0)
				(3, 1.0)
				(4, 1.0)
				(5, 1.0)
				(6, 1.0)
				(7, 1.0)
				(8, 1.0)
				(9, 1.0)
				(10, 1.0)
				(11, 1.0)
				(12, 1.0)
				(13, 1.0)
				(14, 1.0)
				(15, 1.0)
				(16, 1.0)
				(17, 1.0)
				(18, 1.0)
				(19, 1.0)
				(20, 1.0)				
			};
			
			\addplot[
			color=black,
			mark=star,
			mark size=2pt
			]
			coordinates {
				(1, 0.6637956091484005)
				(2, 0.6734838926660947)
				(3, 0.6803263677182906)
				(4, 0.6852445571203793)
				(5, 0.6917828894933222)
				(6, 0.6948495350435676)
				(7, 0.6986076867126946)
				(8, 0.7015228808288303)
				(9, 0.7041242629121558)
				(10, 0.7071069216240197)
				(11, 0.709907593907093)
				(12, 0.7117661335624205)
				(13, 0.7144395293884609)
				(14, 0.715957410159584)
				(15, 0.7177231343205543)
				(16, 0.7196405240618247)
				(17, 0.7216232495161743)
				(18, 0.7230344130303549)
				(19, 0.7248822616699586)
				(20, 0.7265226408200475)				
			};
			
			\addplot[
			color=cyan,
			mark=triangle,
			mark size=2pt
			]
			coordinates{(1, 1.0)
				(2, 0.8437718830701769)
				(3, 0.9002503084562926)
				(4, 0.8525951293066725)
				(5, 0.8849007498435996)
				(6, 0.8575422008155946)
				(7, 0.8802696892651928)
				(8, 0.861211040372987)
				(9, 0.8788853211672494)
				(10, 0.8648245818604862)
				(11, 0.8785930669518535)
				(12, 0.8672589151780755)
				(13, 0.8791476677899444)
				(14, 0.8699789065194989)
				(15, 0.8797998715102773)
				(16, 0.8715738776770809)
				(17, 0.8805848430082662)
				(18, 0.8732309819576819)
				(19, 0.8812542772473494)
				(20, 0.8749439687461352)
				};
			
			\addplot[
			color=brown,
			mark=square,
			mark size=1.5pt
			]
			coordinates{(1, 0.6638823104467751)
				(2, 0.6735774349482342)
				(3, 0.6801879929626223)
				(4, 0.6861766088440103)
				(5, 0.6904013997895555)
				(6, 0.6945672344483165)
				(7, 0.6982041639204912)
				(8, 0.7012809314614449)
				(9, 0.704189074455143)
				(10, 0.7068187588028761)
				(11, 0.7092405020800028)
				(12, 0.7115890850010586)
				(13, 0.7137774225689693)
				(14, 0.7157000270262805)
				(15, 0.7177566177600958)
				(16, 0.7196055720503408)
				(17, 0.7214494855245566)
				(18, 0.7231811476436737)
				(19, 0.7247693254416977)
				(20, 0.7264538760763735)
				};
			
			\end{axis}
			\end{tikzpicture}\label{fig:gl_2_mu}}
		\hfil
		\subfloat[{LF, Metric: \hyperref[sec:cust_metric]{${\mu_{\phi}}$}}]{\pgfplotsset{width=0.29\textwidth,height=0.19\textwidth,compat=1.9}
			\begin{tikzpicture}
			\begin{axis}[
			xlabel={${k}$},
			ylabel={\hyperref[sec:cust_metric]{${\mu_{\phi}}$}},
			xmin=0, xmax=20,
			ymin=0, ymax=1,
			xtick={0,5,10,15,20},
			ytick={0,0.2,0.4,0.6,0.8,1},
			ymajorgrids=true,
			grid style=dashed,
			ylabel style={rotate=-90}
			]
			\addplot[
			color=blue,
			mark=o,
			mark size=2pt
			]
			coordinates {(1, 1.0)
				(2, 1.0)
				(3, 1.0)
				(4, 1.0)
				(5, 1.0)
				(6, 1.0)
				(7, 1.0)
				(8, 1.0)
				(9, 1.0)
				(10, 0.9466840403677732)
				(11, 0.9573332123472116)
				(12, 0.9645741314748756)
				(13, 0.9698975267372104)
				(14, 0.9740279226029401)
				(15, 0.9772794194615271)
				(16, 0.9799094042412378)
				(17, 0.9820742241605919)
				(18, 0.9838885500429514)
				(19, 0.9554355420796439)
				(20, 0.9617003481874509)
				
			};
			
			\addplot[
			color=red,
			mark=x,
			mark size=2pt
			]
			coordinates {(1, 1.0)
				(2, 1.0)
				(3, 1.0)
				(4, 1.0)
				(5, 1.0)
				(6, 1.0)
				(7, 1.0)
				(8, 1.0)
				(9, 1.0)
				(10, 1.0)
				(11, 1.0)
				(12, 1.0)
				(13, 1.0)
				(14, 1.0)
				(15, 1.0)
				(16, 1.0)
				(17, 1.0)
				(18, 1.0)
				(19, 1.0)
				(20, 1.0)
				
			};
			
			\addplot[
			color=black,
			mark=star,
			mark size=2pt
			]
			coordinates {(1, 0.17832003461211382)
				(2, 0.1769745910441589)
				(3, 0.18364845281871542)
				(4, 0.1841895916494282)
				(5, 0.18716048101980956)
				(6, 0.1879339046867482)
				(7, 0.1875864480744251)
				(8, 0.1924937399023132)
				(9, 0.19356422751759575)
				(10, 0.19288994148103977)
				(11, 0.19320626765394686)
				(12, 0.19599057436374165)
				(13, 0.1960350508783371)
				(14, 0.1935376665189075)
				(15, 0.1959156322648999)
				(16, 0.19651852529879993)
				(17, 0.19719627944692023)
				(18, 0.19754206328322177)
				(19, 0.1983699699676643)
				(20, 0.20030689815674965)
				
			};
			\addplot[
			color=cyan,
			mark=triangle,
			mark size=2pt
			]
			coordinates{(1, 1.0)
				(2, 0.6045192389986451)
				(3, 0.7370932601748931)
				(4, 0.6071151291986662)
				(5, 0.689644383398613)
				(6, 0.6084278402625126)
				(7, 0.6685499464377781)
				(8, 0.6084280948091991)
				(9, 0.6572932641567436)
				(10, 0.613001853003144)
				(11, 0.6520584462444754)
				(12, 0.616282071561378)
				(13, 0.6471338128418598)
				(14, 0.6176505112658482)
				(15, 0.6444576007876894)
				(16, 0.6175804374172676)
				(17, 0.6423740132205189)
				(18, 0.6193769176398086)
				(19, 0.6402497299829498)
				(20, 0.6196171909304439)
				};
			
			\addplot[
			color=brown,
			mark=square,
			mark size=1.5pt
			]
			coordinates{(1, 0.14144716232622817)
				(2, 0.1576336974317044)
				(3, 0.16909076109538335)
				(4, 0.1762468780451565)
				(5, 0.18151123293026242)
				(6, 0.18513567898793606)
				(7, 0.18846899439490333)
				(8, 0.19044440742984214)
				(9, 0.19311347284536098)
				(10, 0.1890780307160162)
				(11, 0.18835547301135933)
				(12, 0.18909402200674397)
				(13, 0.18999745558322234)
				(14, 0.19170751278223744)
				(15, 0.19367699862136153)
				(16, 0.19540552618100193)
				(17, 0.1966552218400394)
				(18, 0.19822357163557078)
				(19, 0.19799322388524554)
				(20, 0.1969158476579335)
				};
			
			\end{axis}
			\end{tikzpicture}\label{fig:lf_mu}}
		\vfil
		\subfloat[{GL-CUSTOM, Metric: \hyperref[sec:cust_metric]{${std_\phi}$}}]{\pgfplotsset{width=0.29\textwidth,height=0.19\textwidth,compat=1.9}
			\begin{tikzpicture}
			\begin{axis}[
			xlabel={${k}$},
			ylabel={\hyperref[sec:cust_metric]{${std_\phi}$}},
			xmin=0, xmax=20,
			ymin=0, ymax=0.4,
			xtick={0,5,10,15,20},
			ymajorgrids=true,
			grid style=dashed,
			ylabel style={rotate=-90}
			]
			\addplot[
			color=blue,
			mark=o,
			mark size=2pt
			]
			coordinates {(1, 0.37830805256560884)
				(2, 0.25410394502765443)
				(3, 0.19680571440598094)
				(4, 0.1685207037625616)
				(5, 0.14784959761005217)
				(6, 0.13283238851616477)
				(7, 0.12020286728540513)
				(8, 0.10983712613500922)
				(9, 0.10065168951059378)
				(10, 0.09294555455170196)
				(11, 0.08502314275421591)
				(12, 0.07867193271390735)
				(13, 0.07266810049095368)
				(14, 0.06734393245180494)
				(15, 0.06433088575596334)
				(16, 0.05980036681315315)
				(17, 0.05590994260945399)
				(18, 0.0523305182720453)
				(19, 0.04929585132780681)
				(20, 0.046208732752894666)				
			};
			
			\addplot[
			color=red,
			mark=x,
			mark size=2pt
			]
			coordinates {(1, 0.0)
				(2, 0.0)
				(3, 0.0)
				(4, 0.0)
				(5, 0.0)
				(6, 0.0)
				(7, 0.0)
				(8, 0.0)
				(9, 0.0)
				(10, 0.0)
				(11, 0.0)
				(12, 0.0)
				(13, 0.0)
				(14, 0.0)
				(15, 0.0)
				(16, 0.0)
				(17, 0.0)
				(18, 0.0)
				(19, 0.0)
				(20, 0.0)
					
			};
			
			\addplot[
			color=black,
			mark=star,
			mark size=2pt
			]
			coordinates {(1, 0.0665792494115164)
				(2, 0.06642641249853415)
				(3, 0.06271293333356061)
				(4, 0.060687802843674345)
				(5, 0.057638020502885884)
				(6, 0.06075864205196455)
				(7, 0.059910252869592395)
				(8, 0.06282476435514515)
				(9, 0.0609461238254342)
				(10, 0.06067960737449011)
				(11, 0.0604393763477814)
				(12, 0.06238359751519571)
				(13, 0.060559813966399395)
				(14, 0.061685578127072614)
				(15, 0.06298998883027536)
				(16, 0.06154427038269326)
				(17, 0.06376173371884339)
				(18, 0.06498463186010953)
				(19, 0.06448367750694804)
				(20, 0.06439936476862508)
							
			};
			\addplot[
			color=cyan,
			mark=triangle,
			mark size=2pt
			]
			coordinates{(1, 0.0)
				(2, 0.10523668302215644)
				(3, 0.06588255029425882)
				(4, 0.08816376711883124)
				(5, 0.0685737162746419)
				(6, 0.08166440901180695)
				(7, 0.06967300441086495)
				(8, 0.07894972155891818)
				(9, 0.06808644607623329)
				(10, 0.07468013952200325)
				(11, 0.06553375184060456)
				(12, 0.07050701091507944)
				(13, 0.06308400996891075)
				(14, 0.06712670105419105)
				(15, 0.06053138310106066)
				(16, 0.06409027715486966)
				(17, 0.05878676570160378)
				(18, 0.06189616519034378)
				(19, 0.05730921374086472)
				(20, 0.05982977327890879)
				};
			
			\addplot[
			color=brown,
			mark=square,
			mark size=1.5pt
			]
			coordinates{(1, 0.07182909832403742)
				(2, 0.06303971523250504)
				(3, 0.0593129287515427)
				(4, 0.05923849412907048)
				(5, 0.0597787920523159)
				(6, 0.059224894811932806)
				(7, 0.059109996911624885)
				(8, 0.05987856835017233)
				(9, 0.060266665707672404)
				(10, 0.06028987682584633)
				(11, 0.06034253584176401)
				(12, 0.06119892214984695)
				(13, 0.061841909261809956)
				(14, 0.061759826343640266)
				(15, 0.06179997462516536)
				(16, 0.06216537027723421)
				(17, 0.06308000171156664)
				(18, 0.06355977724255803)
				(19, 0.0636840881655898)
				(20, 0.06402622514477738)
				};
			
			\end{axis}
			\end{tikzpicture}\label{fig:gl_1_sigma}}
		\hfil
		\subfloat[{GL-FACT, Metric: \hyperref[sec:cust_metric]{${std_\phi}$}}]{\pgfplotsset{width=0.29\textwidth,height=0.19\textwidth,compat=1.9}
			\begin{tikzpicture}
			\begin{axis}[
			xlabel={${k}$},
			ylabel={\hyperref[sec:cust_metric]{${std_\phi}$}},
			xmin=0, xmax=20,
			ymin=0, ymax=0.12,
			xtick={0,5,10,15,20},
			ytick={0,0.05,0.10},
			ymajorgrids=true,
			grid style=dashed,
			ylabel style={rotate=-90}
			]
			\addplot[
			color=blue,
			mark=o,
			mark size=2pt
			]
			coordinates {(1, 0.08399702240460524)
				(2, 0.05436271919095607)
				(3, 0.04147625853832062)
				(4, 0.034748350697647985)
				(5, 0.029786123763427416)
				(6, 0.026387935487853478)
				(7, 0.023939164657281208)
				(8, 0.02219728309393469)
				(9, 0.020433298984965105)
				(10, 0.01904127405967224)
				(11, 0.017903688940636382)
				(12, 0.016910073565303354)
				(13, 0.01608358767807022)
				(14, 0.015341726729430199)
				(15, 0.01473696988161612)
				(16, 0.014186237974498379)
				(17, 0.013517023459539142)
				(18, 0.012981677508432696)
				(19, 0.012594477881743952)
				(20, 0.012173784424625316)
				
			};
			
			\addplot[
			color=red,
			mark=x,
			mark size=2pt
			]
			coordinates {(1, 0.0)
				(2, 0.0)
				(3, 0.0)
				(4, 0.0)
				(5, 0.0)
				(6, 0.0)
				(7, 0.0)
				(8, 0.0)
				(9, 0.0)
				(10, 0.0)
				(11, 0.0)
				(12, 0.0)
				(13, 0.0)
				(14, 0.0)
				(15, 0.0)
				(16, 0.0)
				(17, 0.0)
				(18, 0.0)
				(19, 0.0)
				(20, 0.0)
								
			};
			
			\addplot[
			color=black,
			mark=star,
			mark size=2pt
			]
			coordinates {(1, 0.1153001122599757)
				(2, 0.08402134976303625)
				(3, 0.07031821946664513)
				(4, 0.06159441577580742)
				(5, 0.05536246474305281)
				(6, 0.05169877880920397)
				(7, 0.04855659780428536)
				(8, 0.046157236527497585)
				(9, 0.043037196318076634)
				(10, 0.0423901204460739)
				(11, 0.04047684064532946)
				(12, 0.03918725149118922)
				(13, 0.0373408582032027)
				(14, 0.03657097307479459)
				(15, 0.036039312447793126)
				(16, 0.035568864568369295)
				(17, 0.034496914669265105)
				(18, 0.03353604096856207)
				(19, 0.03268238920676137)
				(20, 0.03240727416757377)

			};
			
			\addplot[
			color=cyan,
			mark=triangle,
			mark size=2pt
			]
			coordinates{(1, 0.0)
				(2, 0.05710351368759226)
				(3, 0.03812489941231977)
				(4, 0.04145536943608316)
				(5, 0.03305394593115184)
				(6, 0.034319363320058825)
				(7, 0.029281322895173413)
				(8, 0.030134888224788527)
				(9, 0.026964033269695705)
				(10, 0.027409919783979247)
				(11, 0.02485400412136951)
				(12, 0.02509775733350031)
				(13, 0.023042687426656567)
				(14, 0.023343698672126677)
				(15, 0.022004783279296728)
				(16, 0.022273688683520733)
				(17, 0.0212194959495044)
				(18, 0.021257703405436028)
				(19, 0.020063266908055585)
				(20, 0.02029022044520764)
				};
			
			\addplot[
			color=brown,
			mark=square,
			mark size=1.5pt
			]
			coordinates{(1, 0.11387422047678955)
				(2, 0.08403420729543372)
				(3, 0.07044462131678358)
				(4, 0.06241020200447165)
				(5, 0.05677074894746044)
				(6, 0.05259923384523701)
				(7, 0.04944074059587094)
				(8, 0.04673404082302104)
				(9, 0.04459528667217272)
				(10, 0.04263978200055096)
				(11, 0.04117743750520475)
				(12, 0.03967059107626005)
				(13, 0.03830643041352416)
				(14, 0.03729821408319296)
				(15, 0.03637271020859325)
				(16, 0.03529737980490692)
				(17, 0.03440632291268206)
				(18, 0.033585749611649715)
				(19, 0.03283481401691028)
				(20, 0.03227060307439308)
				};
			
			\end{axis}
			\end{tikzpicture}\label{fig:gl_2_sigma}}
		\hfil
		\subfloat[{LF, Metric: \hyperref[sec:cust_metric]{${std_\phi}$}}]{\pgfplotsset{width=0.29\textwidth,height=0.19\textwidth,compat=1.9}
			\begin{tikzpicture}
			\begin{axis}[
			xlabel={${k}$},
			ylabel={\hyperref[sec:cust_metric]{${std_\phi}$}},
			xmin=0, xmax=20,
			ymin=0, ymax=0.3,
			xtick={0,5,10,15,20},
			ytick={0,0.1,0.2,0.3},
			ymajorgrids=true,
			grid style=dashed,
			ylabel style={rotate=-90}
			]
			\addplot[
			color=blue,
			mark=o,
			mark size=2pt
			]
			coordinates {(1, 0.0)
				(2, 0.0)
				(3, 0.0)
				(4, 0.0)
				(5, 0.0)
				(6, 0.0)
				(7, 0.0)
				(8, 0.0)
				(9, 0.0)
				(10, 0.02488225540664339)
				(11, 0.021099297022888097)
				(12, 0.018581836098998605)
				(13, 0.016676050699629807)
				(14, 0.015159058205829136)
				(15, 0.013896429905161994)
				(16, 0.012824286260091034)
				(17, 0.011902275955034826)
				(18, 0.011105539613363477)
				(19, 0.015687238458487282)
				(20, 0.014096324416347708)
				
			};
			
			\addplot[
			color=red,
			mark=x,
			mark size=2pt
			]
			coordinates {(1, 0.0)
				(2, 0.0)
				(3, 0.0)
				(4, 0.0)
				(5, 0.0)
				(6, 0.0)
				(7, 0.0)
				(8, 0.0)
				(9, 0.0)
				(10, 0.0)
				(11, 0.0)
				(12, 0.0)
				(13, 0.0)
				(14, 0.0)
				(15, 0.0)
				(16, 0.0)
				(17, 0.0)
				(18, 0.0)
				(19, 0.0)
				(20, 0.0)
				
			};
			
			\addplot[
			color=black,
			mark=star,
			mark size=2pt
			]
			coordinates {(1, 0.2575423373515667)
				(2, 0.20569419196264266)
				(3, 0.19136184104992532)
				(4, 0.18176403425663273)
				(5, 0.1745275803792753)
				(6, 0.17260452124759076)
				(7, 0.16936880075372318)
				(8, 0.16498988342394214)
				(9, 0.16127830629709414)
				(10, 0.16312694493455873)
				(11, 0.16228692055918909)
				(12, 0.16051243296644954)
				(13, 0.15993010498879995)
				(14, 0.1591104695231111)
				(15, 0.15677233752737132)
				(16, 0.15889714291012752)
				(17, 0.15680638724449858)
				(18, 0.15670381750110401)
				(19, 0.1561156817656419)
				(20, 0.15674307532191106)
				
			};
			
			\addplot[
			color=cyan,
			mark=triangle,
			mark size=2pt
			]
			coordinates{(1, 0.0)
				(2, 0.1270219374624834)
				(3, 0.08856027149307515)
				(4, 0.10651432112260012)
				(5, 0.08620413695292845)
				(6, 0.0944668583196028)
				(7, 0.08168001330995907)
				(8, 0.09034951456813663)
				(9, 0.07972546300853252)
				(10, 0.08780189353595864)
				(11, 0.07877015462020656)
				(12, 0.084195235016906)
				(13, 0.07662875332097971)
				(14, 0.08382825228826962)
				(15, 0.0777393945704526)
				(16, 0.0817489553094493)
				(17, 0.07691157122728591)
				(18, 0.0808469277424271)
				(19, 0.07577794693346765)
				(20, 0.08022440140232145)
				};
			
			\addplot[
			color=brown,
			mark=square,
			mark size=1.5pt
			]
			coordinates{(1, 0.24221177815800354)
				(2, 0.1952505224108675)
				(3, 0.17964613717756578)
				(4, 0.1721199793052581)
				(5, 0.1698628184104722)
				(6, 0.16658294881110647)
				(7, 0.16381746911487488)
				(8, 0.1624525860575433)
				(9, 0.16182377705536785)
				(10, 0.15894753331713007)
				(11, 0.15832556786484017)
				(12, 0.1572826280022425)
				(13, 0.15727696169426678)
				(14, 0.15744622568399577)
				(15, 0.15725029107566524)
				(16, 0.15711631770155113)
				(17, 0.15754089674536897)
				(18, 0.15739808855625487)
				(19, 0.15594792930102305)
				(20, 0.1550381728460963)
				};

			\end{axis}
			\end{tikzpicture}\label{fig:lf_sigma}}
		\vfil
		\subfloat{\pgfplotsset{width=.7\textwidth,compat=1.9}
			\begin{tikzpicture}
			\begin{customlegend}[legend entries={{FairRec},{Top-$k$},{Random-$k$},{Mixed-$k$},{PR-$k$}},legend columns=5,legend style={/tikz/every even column/.append style={column sep=0.5cm}}]
			\addlegendimage{blue,mark=o,sharp plot}
			\addlegendimage{red,mark=x,sharp plot}
			\addlegendimage{black,mark=star,sharp plot}
			\addlegendimage{cyan,mark=triangle,sharp plot}
			\addlegendimage{brown,mark=square,sharp plot}
			\end{customlegend}
			\end{tikzpicture}}
	}\caption{Customer-Side Performances with MMS Guarantee.First row: mean average envy (\hyperref[metric_Y]{$Y$}). Second row: mean customer utility (\hyperref[sec:cust_metric]{$\mu_\phi$}). Third row: standard deviation of customer utilities (\hyperref[sec:cust_metric]{$std_\phi$}).}\label{fig:customer_side}
\end{figure*}
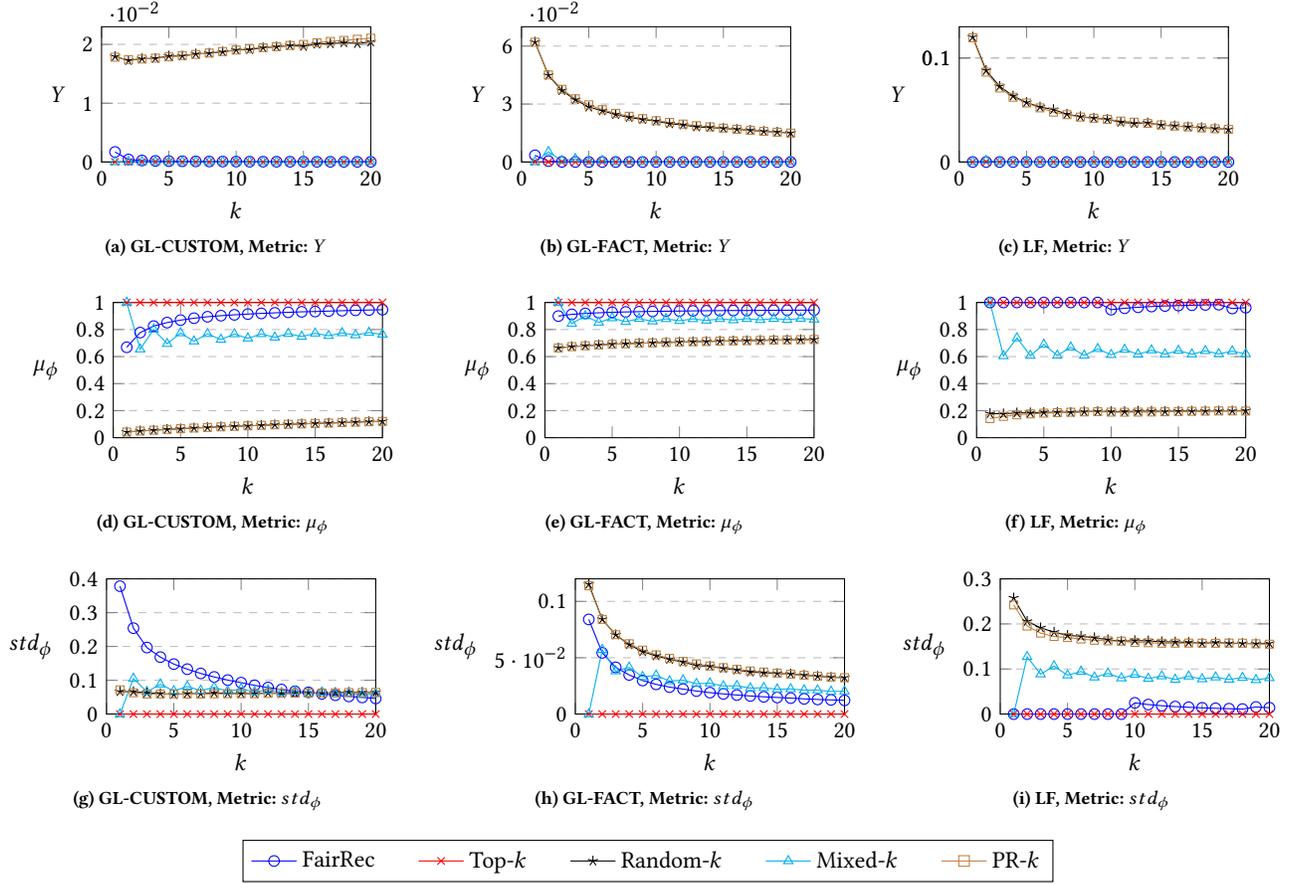
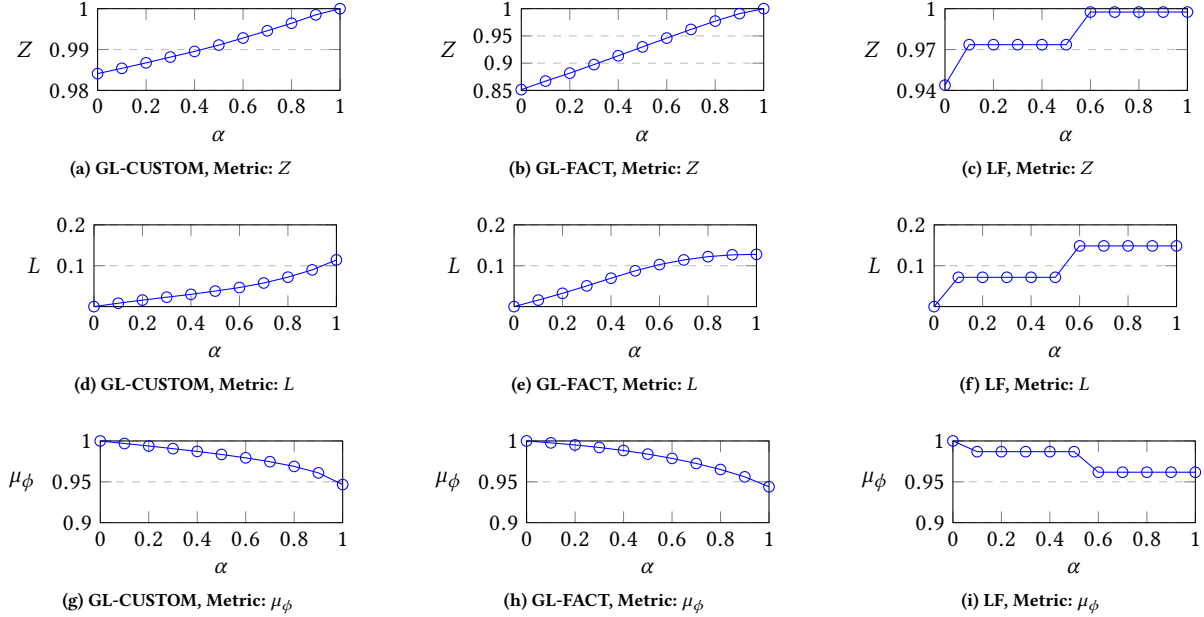
\begin{figure*}[t!]
	\center{
		\subfloat[{GL-CUSTOM, Metric: \hyperref[metric_Z]{$Z$}}]{\pgfplotsset{width=0.27\textwidth,height=0.15\textwidth,compat=1.9}
			\begin{tikzpicture}
			\begin{axis}[
			xlabel={${\alpha}$},
			ylabel={\hyperref[metric_Z]{$Z$}},
			xmin=0, xmax=1,
			ymin=0.98, ymax=1,
			xtick={0,0.2,0.4,0.6,0.8,1},
			ytick={0.98,0.99,1},
			ymajorgrids=true,
			grid style=dashed,
			ylabel style={rotate=-90}
			]
			\addplot[
			color=blue,
			mark=o,
			mark size=2pt
			]
			coordinates {(0,0.984106342444825)
				(0.1, 0.9853980754015105)
				(0.2, 0.9867420331086655)
				(0.3, 0.988154210874973)
				(0.4, 0.9895327983447276)
				(0.5, 0.9910921580190593)
				(0.6, 0.9928205788984502)
				(0.7, 0.9945906159004281)
				(0.8, 0.9964309446220851)
				(0.9, 0.9984910648650908)
				(1, 0.9999991439597733)
				
			};

			\end{axis}
			\end{tikzpicture}\label{fig:gl_1_alpha_Z}}
		\hfil
		\subfloat[{GL-FACT, Metric: \hyperref[metric_Z]{$Z$}}]{\pgfplotsset{width=0.27\textwidth,height=0.15\textwidth,compat=1.9}
			\begin{tikzpicture}
			\begin{axis}[
			xlabel={${\alpha}$},
			ylabel={\hyperref[metric_Z]{$Z$}},
			xmin=0, xmax=1,
			ymin=0.85, ymax=1,
			xtick={0,0.2,0.4,0.6,0.8,1},
			ytick={0.85,0.9,0.95,1},
			ymajorgrids=true,
			grid style=dashed,
			ylabel style={rotate=-90}
			]
			\addplot[
			color=blue,
			mark=o,
			mark size=2pt
			]
			coordinates {(0,0.8514541177377445)
				(0.1, 0.8670708016128668)
				(0.2, 0.8816758440831862)
				(0.3, 0.8973834064232564)
				(0.4, 0.9135487739052681)
				(0.5, 0.9298708159771711)
				(0.6, 0.946002520981026)
				(0.7, 0.9619162896090064)
				(0.8, 0.9772018655579906)
				(0.9, 0.990873158852095)
				(1, 0.9999953633843445)
				
			};

			\end{axis}
			\end{tikzpicture}\label{fig:gl_2_alpha_Z}}
		\hfil
		\subfloat[{LF, Metric: \hyperref[metric_Z]{$Z$}}]{\pgfplotsset{width=0.27\textwidth,height=0.15\textwidth,compat=1.9}
			\begin{tikzpicture}
			\begin{axis}[
			xlabel={${\alpha}$},
			ylabel={\hyperref[metric_Z]{$Z$}},
			xmin=0, xmax=1,
			ymin=0.94, ymax=1,
			xtick={0,0.2,0.4,0.6,0.8,1},
			ytick={0.94,0.97,1},
			ymajorgrids=true,
			grid style=dashed,
			ylabel style={rotate=-90}
			]
			\addplot[
			color=blue,
			mark=o,
			mark size=2pt
			]
			coordinates {(0,0.9438539852110053)
				(0.1, 0.9735613891721784)
				(0.2, 0.9735613891721784)
				(0.3, 0.9735613891721784)
				(0.4, 0.9735613891721784)
				(0.5, 0.9735613891721784)
				(0.6, 0.9975282401963197)
				(0.7, 0.9975282401963197)
				(0.8, 0.9975282401963197)
				(0.9, 0.9975282401963197)
				(1, 0.9975310687214431)
				
			};

			\end{axis}
			\end{tikzpicture}\label{fig:lf_alpha_Z}}
		\vfil
		\subfloat[{GL-CUSTOM, Metric: \hyperref[metric_L]{$L$}}]{\pgfplotsset{width=0.27\textwidth,height=0.15\textwidth,compat=1.9}
			\begin{tikzpicture}
			\begin{axis}[
			xlabel={${\alpha}$},
			ylabel={\hyperref[metric_L]{$L$}},
			xmin=0, xmax=1,
			ymin=0, ymax=0.2,
			xtick={0,0.2,0.4,0.6,0.8,1},
			ytick={0.1,0.2},
			ymajorgrids=true,
			grid style=dashed,
			ylabel style={rotate=-90},
			]
			\addplot[
			color=blue,
			mark=o,
			mark size=2pt
			]
			coordinates {(0, 0)
				(0.1, 0.008271704655220072)
				(0.2, 0.015912717169677)
				(0.3, 0.022996430687310283)
				(0.4, 0.02994441402734677)
				(0.5, 0.03800347938805518)
				(0.6, 0.04662348783773038)
				(0.7, 0.057734241855776705)
				(0.8, 0.07195809124382144)
				(0.9, 0.08983590644152349)
				(1, 0.11420550431363932)

			};

			\end{axis}
			\end{tikzpicture}\label{fig:gl_1_alpha_L}}
		\hfil
		\subfloat[{GL-FACT, Metric: \hyperref[metric_L]{$L$}}]{\pgfplotsset{width=0.27\textwidth,height=0.15\textwidth,compat=1.9}
			\begin{tikzpicture}
			\begin{axis}[
			xlabel={${\alpha}$},
			ylabel={\hyperref[metric_L]{$L$}},
			xmin=0, xmax=1,
			ymin=0, ymax=0.2,
			xtick={0,0.2,0.4,0.6,0.8,1},
			ytick={0.1,0.2},
			ymajorgrids=true,
			grid style=dashed,
			ylabel style={rotate=-90},
			]
			\addplot[
			color=blue,
			mark=o,
			mark size=2pt
			]
			coordinates {(0, 0)
				(0.1, 0.01592092699103491)
				(0.2, 0.03247747188548497)
				(0.3, 0.050639864451723526)
				(0.4, 0.06943283276909268)
				(0.5, 0.08760013721219506)
				(0.6, 0.10285231536136971)
				(0.7, 0.1141603658657398)
				(0.8, 0.12217723285984362)
				(0.9, 0.1266171694638939)
				(1, 0.12787155355483099)

			};

			\end{axis}
			\end{tikzpicture}\label{fig:gl_2_alpha_L}}
		\hfil
		\subfloat[{LF, Metric: \hyperref[metric_L]{$L$}}]{\pgfplotsset{width=0.27\textwidth,height=0.15\textwidth,compat=1.9}
			\begin{tikzpicture}
			\begin{axis}[
			xlabel={${\alpha}$},
			ylabel={\hyperref[metric_L]{$L$}},
			xmin=0, xmax=1,
			ymin=0, ymax=0.2,
			xtick={0,0.2,0.4,0.6,0.8,1},
			ytick={0.1,0.2},
			ymajorgrids=true,
			grid style=dashed,
			ylabel style={rotate=-90},
			]
			\addplot[
			color=blue,
			mark=o,
			mark size=2pt
			]
			coordinates {(0, 0)
				(0.1, 0.07165146245805837)
				(0.2, 0.07165146245805837)
				(0.3, 0.07165146245805837)
				(0.4, 0.07165146245805837)
				(0.5, 0.07165146245805837)
				(0.6, 0.14848261833001544)
				(0.7, 0.14848261833001544)
				(0.8, 0.14848261833001544)
				(0.9, 0.14848261833001544)
				(1, 0.14851570211706927)		
			};

			\end{axis}
			\end{tikzpicture}\label{fig:lf_alpha_L}}
		\vfil
		\subfloat[{GL-CUSTOM, Metric: \hyperref[sec:cust_metric]{$\mu_{\phi}$}}]{\pgfplotsset{width=0.27\textwidth,height=0.15\textwidth,compat=1.9}
			\begin{tikzpicture}
			\begin{axis}[
			xlabel={${\alpha}$},
			ylabel={\hyperref[sec:cust_metric]{$\mu_{\phi}$}},
			xmin=0, xmax=1,
			ymin=0.9, ymax=1,
			xtick={0,0.2,0.4,0.6,0.8,1},
			ytick={0.9,0.95,1},
			ymajorgrids=true,
			grid style=dashed,
			ylabel style={rotate=-90},
			]
			\addplot[
			color=blue,
			mark=o,
			mark size=2pt
			]
			coordinates {(0, 1)
				(0.1, 0.996871285403606)
				(0.2, 0.9938160401740173)
				(0.3, 0.9905835142445624)
				(0.4, 0.9872141916055803)
				(0.5, 0.9834836402522825)
				(0.6, 0.9793293369286161)
				(0.7, 0.9746136807880204)
				(0.8, 0.9688379921300452)
				(0.9, 0.9609441835903161)
				(1, 0.9466278029121776)			
				
			};

			\end{axis}
			\end{tikzpicture}\label{fig:gl_1_alpha_mu}}
		\hfil
		\subfloat[{GL-FACT, Metric: \hyperref[sec:cust_metric]{$\mu_{\phi}$}}]{\pgfplotsset{width=0.27\textwidth,height=0.15\textwidth,compat=1.9}
			\begin{tikzpicture}
			\begin{axis}[
			xlabel={${\alpha}$},
			ylabel={\hyperref[sec:cust_metric]{$\mu_{\phi}$}},
			xmin=0, xmax=1,
			ymin=0.9, ymax=1,
			xtick={0,0.2,0.4,0.6,0.8,1},
			ytick={0.9,0.95,1},
			ymajorgrids=true,
			grid style=dashed,
			ylabel style={rotate=-90},
			]
			\addplot[
			color=blue,
			mark=o,
			mark size=2pt
			]
			coordinates {(0, 1)
				(0.1, 0.9976423289610249)
				(0.2, 0.9951127244398144)
				(0.3, 0.992029062972915)
				(0.4, 0.9883351422975385)
				(0.5, 0.98390391887269)
				(0.6, 0.9786611372211416)
				(0.7, 0.9724443569939579)
				(0.8, 0.9650758419710358)
				(0.9, 0.9560885282850516)
				(1, 0.9439498822796131)			
				
			};

			\end{axis}
			\end{tikzpicture}\label{fig:gl_2_alpha_mu}}
		\hfil
		\subfloat[{LF, Metric: \hyperref[sec:cust_metric]{$\mu_{\phi}$}}]{\pgfplotsset{width=0.27\textwidth,height=0.15\textwidth,compat=1.9}
			\begin{tikzpicture}
			\begin{axis}[
			xlabel={${\alpha}$},
			ylabel={\hyperref[sec:cust_metric]{$\mu_{\phi}$}},
			xmin=0, xmax=1,
			ymin=0.9, ymax=1,
			xtick={0,0.2,0.4,0.6,0.8,1},
			ytick={0.9,0.95,1},
			ymajorgrids=true,
			grid style=dashed,
			ylabel style={rotate=-90},
			]
			\addplot[
			color=blue,
			mark=o,
			mark size=2pt
			]
			coordinates {(0, 1)
				(0.1, 0.986879107905887)
				(0.2, 0.986879107905887)
				(0.3, 0.986879107905887)
				(0.4, 0.986879107905887)
				(0.5, 0.986879107905887)
				(0.6, 0.9617520116271875)
				(0.7, 0.9617520116271875)
				(0.8, 0.9617520116271875)
				(0.9, 0.9617520116271875)				
				(1, 0.9617003481874509)			
				
			};

			\end{axis}
			\end{tikzpicture}\label{fig:lf_alpha_mu}}
	}\caption{Performances with $\overline{E}=\alpha$MMS guarantee for $k=20$. Plots show that higher exposure guarantee can achieve better producer fairness but can cause exposure loss for very popular producers and utility loss for the customers.}\label{fig:alpha_performances}
	\vspace{-2mm}
\end{figure*}
\section{experimental evaluation}\label{experiments}
{\bf Experimental Setup and Baselines:}
We run the proposed {\em FairRec} algorithm (\cref{algorithm}) for all the datasets (as listed in \cref{dataset}) considering different values of the recommendation-size $k$.
For comparison, we use the following methods as baselines:\\
{\bf (1) Top-$k$:} recommending the top-$k$ relevant products,\\
{\bf (2) Random-$k$:} randomly recommending $k$ products,\\
{\bf (3) Mixed-$k$:} choosing  top $\left\lceil\frac{k}{2}\right\rceil$ relevant products at first and then the remaining $\big(k-\left\lceil\frac{k}{2}\right\rceil\big)$ randomly.\\
{\bf (4) Poorest-$k$ (PR-$k$):} this is a producer-centric method where $k$ least exposed products are recommended to each customer in a round robin manner.		

We run two sets of experiments. First, we set the exposure guarantee $\overline{E}=$MMS (in \cref{mms_experiments}). Next, we set lower exposure guarantees i.e., by considering $\overline{E}=\alpha \cdot$MMS where $0 \leq \alpha \leq 1$ (in \cref{alpha_mms_experiments}).  For evaluating \textit{FairRec} and the baselines, we use the following \textit{producer-side} and \textit{customer-side} metrics.
		\subsubsection{\bf Producer-Side Metrics}
		The evaluation metrics for capturing the fairness and efficiency among the producers are:
		
		\noindent \textbf{\textit{Fraction of Satisfied Producers ($H$):}}
		We call a producer satisfied iff its exposure is more than the minimum exposure guarantee $\overline{E}$. The fraction of \textit{satisfied} producers can be calculated as below.
		\begin{equation}\small
			H=\frac{1}{|P|}\sum_{p \in P}\mathbbm{1}_{E_p\geq \overline{E}}\label{metric_H}
		\end{equation}
		{$\mathbbm{1}_{E_p\geq \overline{E}}$ is $1$ if $E_p\geq \overline{E}$, otherwise $0$.} The value of $H$ ranges between $0$ and $1$. The higher the $H$, the fairer is the recommender system to producers.
		
		\noindent \textbf{\textit{Inequality in Producer Exposures ($Z$):}}
		We earlier observed in \cref{motivation} that conventional top-$k$ recommendation causes huge disparity in individual producer exposures.
		To capture how unequal the individual producer exposures are, we employ an entropy like measure as below.
		\begin{equation}\small
			Z=-\sum_{p\in P}\Big(\frac{E_p}{m\times k}\Big)\cdot\log_{n}\Big(\frac{E_p}{m\times k}\Big)\label{metric_Z}
		\end{equation}
		The range of $Z$ is $[0,1]$;
		when all the producers get equal exposure, then $Z$ becomes $1$.
		The lower the $Z$, the more unequal individual producer exposures are.
		
		\noindent \textbf{\textit{Exposure Loss on Producers ($L$):}} 
		As {\em FairRec} tries to ensure minimum exposure guarantee for all the producers,
		some producers may receive a lower exposure in comparison to what they would have got in top-$k$ recommendations. To capture this, we compute the loss $L$ as the mean amount of impact (loss in exposure) caused by {\em FairRec}, compared to the top-$k$ recommendations.
		\begin{equation}\small
			L^\text{<method>}=\frac{1}{n}\sum_{p \in P}\text{max}\Bigg(\frac{\big(E^\text{top-$k$}_p-E^\text{<method>}_p\big)}{E^\text{top-$k$}_p},0\Bigg)\label{metric_L}
		\end{equation}
		The lower the negative impact, the better is the recommendation algorithm.
		
		\subsubsection{\bf Customer-Side Metrics}\label{sec:cust_metric}
		The evaluation metrics for capturing the fairness and efficiency among the customers are:
		
		\noindent \textbf{\textit{Mean Average Envy ($Y$):}}
		Although {\em FairRec} ensures {\em EF1} guarantee for customers by design, here we capture how effectively this guarantee can reduce overall envy among customers in comparison to the baselines.
		We define the mean average envy as below.
		\begin{equation}\small
			Y=\dfrac{1}{n}\sum_{u\in U}\dfrac{1}{n-1}\sum_{\substack{u'\in U\\ u'\ne u}}\text{envy}(u,u')\label{metric_Y}
		\end{equation}
		where {\small $\text{envy}(u,u')=\text{max}\Big(\big(\phi_u(R_{u'})-\phi_u(R_u)\big),0\Big)$} denoting how much $u$ envies $u'$, which is the extra utility $u$ would have received if she had received the recommendation that had been given to $u'$ ($R_{u'}$) instead of her own allocated recommendation $R_u$. The lower the envy ($Y$), the fairer the recommender system is for the customers.
		
		\noindent \textbf{\textit{Loss and Disparity in Customer Utilities (using {\small $\mu_\phi$, $std_\phi$}):}}
		{\em FairRec} may not allocate the most relevant products to the customers,
		which may introduce a loss in customer utilities. This loss can be captured using the expression {\small $\mu_\phi=\frac{1}{m}\sum_{u\in U}\phi_{u}(R_{u})$}. 
		The lower is the utility loss, the more efficient is the recommender system for the customers. 
		We also calculate the standard deviation of customer utilities, that is, {\small $std_\phi = std_{u\in U} (\phi_u(R_u))$}. The lower the $std$, lesser is the disparity in individual customer utilities.
		\subsection{Experiments with MMS Guarantee}\label{mms_experiments}
		Here we test {\em FairRec} with exposure guarantee $\overline{E}=\left\lfloor\frac{mk}{n}\right\rfloor=$MMS (or $\alpha=1$), recommendation size $k$ in $1$ to $20$, and discuss the results. 
		\subsubsection{\bf Producer-Side Results}\label{producer_side}
		All producer side results are plotted in Figure-\ref{fig:producer_side_mms}.\\
		{\bf Producer Satisfaction (\hyperref[metric_H]{$H$}):} Figures \ref{fig:gl_1_H}, \ref{fig:gl_2_H}, and \ref{fig:lf_H} show that both {\em FairRec} and PR-$k$ perform the best while top-$k$, mixed-$k$ perform the worst; 
		this is because both {\em FairRec} and PR-$k$ try to ensure larger exposure for producers while top-$k$, mixed-$k$ consider only the preferences of the customers;\\
		{\bf Exposure Inequality (\hyperref[metric_Z]{$Z$}):} 
		Figures \ref{fig:gl_1_Z}, \ref{fig:gl_2_Z}, and \ref{fig:lf_Z} show that PR-$k$ has lowest inequality in exposure while {\em FairRec} and random-$k$ perform similar or slightly less than that;
		on the other hand top-$k$ and mixed-$k$ perform the worst as they are highly customer-centric.\\
		{\bf Exposure Loss (\hyperref[metric_L]{$L$}):}
		Figures \ref{fig:gl_1_L}, \ref{fig:gl_2_L}, and \ref{fig:lf_L} show that random-$k$ and PR-$k$ cause the highest amounts of exposure loss in comparision to top-$k$;
		this is because both of them favor equality in producer exposure (random-$k$ gives equal chance to all producers to be recommended while PR-$k$ tries to increase the exposure of least exposed producer).
		On the other hand, mixed-$k$ causes smaller losses as only up to half of its recommendations are random.
		{\em FairRec} causes only up to $0.2$ fraction or $20$\% loss in exposure in comparison to top-$k$ owing to the intelligent selection approach of {\em FairRec}.
		
		It is worth noticing that MMS for {\em LF} is low (MMS$=0$ for $k<10$, MMS$=1$ for $k\in[10,18]$, MMS$=2$ for $k\in[20,29]$,...).
		MMS is satisfied for all producers until $k=9$;
		but at k=10, MMS is not guaranteed for all producers, and thus, we see a drop in performance at $k=10$ which happens again at $k=19$. 
		Such changes in MMS specific to {\em LF} make its plots different from other datasets.
		In summary, both {\em FairRec}, PR-$k$ perform the best in producer fairness while they cause exposure loss for very popular producers to compensate for the exposure given to less popular producers.
		\subsubsection{\bf Customer-Side Results}\label{customer_side}
		All customer side results are plotted in Figure-\ref{fig:customer_side}.\\
		{\bf Mean Average Envy (\hyperref[metric_Y]{$Y$}):}
		Figures \ref{fig:gl_1_Y}, \ref{fig:gl_2_Y}, and \ref{fig:lf_Y} reveal that top-$k$ causes lowest possible mean average envy among the customers;
		this is because it gives maximum possible utility of $1$ to every customer thereby leaving no chances of envy among customers.
		Other customer-centric mixed-$k$ also shows similarly low envy.
		{\em FairRec} generates very low values of envy which are very comparable to those of top-$k$ here.
		On the other hand both random-$k$ and PR-$k$ cause the highest envy as they do not consider customer preferences at all during recommendation.\\
		{\bf Loss and Disparity in Customer Utility (\hyperref[sec:cust_metric]{$\mu_\phi$},\hyperref[sec:cust_metric]{$std_\phi$}):}
		From Figures \ref{fig:gl_1_mu}, \ref{fig:gl_2_mu}, and \ref{fig:lf_mu}, we see that both random-$k$ and PR-$k$ cause huge loss in customer utility as they neglect customer preferences.
		On the other hand, while mixed-$k$ performs moderately, {\em FairRec} causes very less utility loss and performs almost at par with the customer-centric top-$k$.
		This certifies that {\em FairRec} strikes a good balance between customer utility and producer fairness.  
		The standard deviation plots: Figures-\ref{fig:gl_1_sigma}, \ref{fig:gl_2_sigma}, and \ref{fig:lf_sigma} reveal that for larger sizes of recommendation, random-$k$ and PR-$k$ show large disparities in customer utilities while {\em FairRec} and mixed-$k$ show relatively less  disparities.
		As top-$k$ is customer-centric and provides the maximum utility of $1$ to all the customers, it shows standard deviation of $0$.
		\subsection{Experiments with $\alpha$-MMS Guarantee}\label{alpha_mms_experiments}
		Here we fix $k=20$, and test {\em FairRec} with different values of minimum exposure guarantee i.e., $\overline{E}=\left\lfloor\alpha\cdot \frac{mk}{n}\right\rfloor$ (where $0\leq\!\alpha\!\leq1$) by varying $\alpha$ in between $0$ and $1$ (or in other words varying $\overline{E}$ in between $0$ and MMS);
		although we do not see much change in \hyperref[metric_H]{$H$}, \hyperref[metric_T]{$Y$}, \hyperref[sec:cust_metric]{$std_\phi$}, we do find some interesting insights from the variations observed in other metrics (detailed next).
		
		\noindent \textbf{\textit{(i) Increasing minimum exposure guarantee results in lower inequality in producer exposures.}}
		We observe the direct correlation of  \hyperref[metric_Z]{$Z$}  with the $\alpha$ value  (refer Figures-\ref{fig:gl_1_alpha_Z},\ref{fig:gl_2_alpha_Z},\ref{fig:lf_alpha_Z}).
		High \hyperref[metric_Z]{$Z$} signifies lower inequality.
		
		\noindent \textbf{\textit{(ii) Increasing minimum exposure guarantee can cause higher exposure losses for previously popular producers.}}
		Increasing values of \hyperref[metric_L]{$L$} or mean exposure loss are observed (refer Figures-\ref{fig:gl_1_alpha_L}, \ref{fig:gl_2_alpha_L}, \ref{fig:lf_alpha_L}) for higher $\alpha$.
		
		\noindent \textbf{\textit{(iii) Higher exposure guarantee for producers can negatively impact overall customer utility.}}
		Figures-\ref{fig:gl_1_alpha_mu}, \ref{fig:gl_2_alpha_mu}, \ref{fig:lf_alpha_mu} show that mean customer utility decreases with higher $\alpha$ or exposure guarantee for the producers.
		
		In summary, although a larger exposure guarantee can help platforms achieve better producer fairness, it might hurt the overall customer satisfaction and also the satisfaction of highly popular producers of the platforms.
		Thus, the platforms, who are interested in similar minimum exposure guarantees, should not ignore the above trade-offs.\\
\section{conclusion}\label{discussion}
In this work, we provide a scalable and easily adaptable algorithm that exhibits desired two-sided fairness properties while causing a marginal loss in the overall quality of recommendations.
We establish theoretical guarantees and provide empirical evidence through extensive evaluations of real-world datasets.
Our work can be directly applied to fair recommendation problems in scenarios like mass recommendation/promotion sent through emails, app/web notifications. Though our work considers the offline recommendation scenario where the recommendations are computed for all the registered customers at once, it can also be extended for online recommendation settings by limiting the set of customers to only the active customers at any particular instant. However, developing a more robust realization of the proposed mechanism for a completely online scenario remains future work. Going ahead, we also want to study attention models that can handle position bias~\cite{agarwal2019estimating}, where customers pay more attention to the top-ranked products than the lower-ranked ones.\\

\noindent {\bf Acknowledgements:}
G. K Patro is supported by TCS Research Fellowship.
A. Biswas gratefully acknowledges the support of a Google Ph.D. Fellowship Award.
This research was supported in part by a European Research Council (ERC) Advanced Grant for the project ``Foundations for Fair Social Computing", funded under the EU Horizon 2020 Framework Programme (grant agreement no. 789373).\\

\noindent {\bf Reproducibility:}
All the codes and datasets details are available at \url{https://github.com/gourabkumarpatro/FairRec_www_2020}.


{
	
	\bibliographystyle{ACM-Reference-Format}
	\bibliography{main}
}

\end{document}